%% file: main.tex
\documentclass{article}

\usepackage[accepted]{icml2023}

\usepackage[protrusion=true,expansion=true]{microtype}		

\usepackage{amsmath}
\usepackage{amsthm}
\usepackage{amssymb}
\usepackage{bbm}
\usepackage{mathtools}
\usepackage{mathrsfs}
\mathtoolsset{showonlyrefs}

\newtheorem{thm}{Theorem}[]

\newtheorem{rem}{Remark}[]

\usepackage{thmtools}
\usepackage{thm-restate}

\usepackage{courier} 

\usepackage{lipsum} 

\usepackage{float}
\usepackage[space]{grffile} 

\usepackage{hyperref}
\usepackage{theoremref}

\usepackage{graphicx}
\usepackage{wrapfig}
\usepackage{xcolor}
\definecolor{dark-red}{rgb}{0.4,0.15,0.15}
\definecolor{dark-blue}{rgb}{0,0,0.7}
\hypersetup{
    colorlinks, linkcolor={dark-blue},
    citecolor={dark-blue}, urlcolor={dark-blue}
}

\DeclareMathOperator*{\argmin}{arg\,min}
\DeclareMathOperator*{\argmax}{arg\,max}

\usepackage{multicol}
\usepackage{enumitem}
\usepackage[utf8]{inputenc} 
\usepackage[T1]{fontenc}    
\usepackage{url}            
\usepackage{booktabs}       
\usepackage{nicefrac}       

\icmltitlerunning{Representations and Exploration for Deep Reinforcement Learning using Singular Value Decomposition}

\begin{document}

\twocolumn[
\icmltitle{Representations and Exploration for Deep Reinforcement Learning \\ using Singular Value Decomposition}



\icmlsetsymbol{equal}{*}

\begin{icmlauthorlist}
\icmlauthor{Yash Chandak}{um}
\icmlauthor{Shantanu Thakoor}{dm}
\icmlauthor{Zhaohan Daniel Guo}{dm}
\icmlauthor{Yunhao Tang}{dm}
\\
\icmlauthor{Remi Munos}{dm}
\icmlauthor{Will Dabney}{dm}
\icmlauthor{Diana Borsa}{dm}
\end{icmlauthorlist}

\icmlaffiliation{um}{University of Massachusetts}
\icmlaffiliation{dm}{DeepMind}

\icmlcorrespondingauthor{Yash Chandak}{ychandak@cs.umass.edu}

\icmlkeywords{Machine Learning, ICML}

\vskip 0.3in
]



\printAffiliationsAndNotice{} 


\begin{abstract}
Representation learning and exploration are among the key challenges for any deep reinforcement learning agent.
In this work, we provide a singular value decomposition-based method that can be used to obtain representations that preserve the underlying transition structure in the domain.
Perhaps interestingly, we show that these representations also capture the relative frequency of state visitations, thereby providing an estimate for pseudo-counts for free.
To scale this decomposition method to large-scale domains, we provide an algorithm that never requires building the transition matrix, can make use of deep networks, and also permits mini-batch training.
Further, we draw inspiration from predictive state representations and extend our decomposition method to partially observable environments.
With experiments on multi-task settings with partially observable domains, we show that the proposed method can not only learn useful representation on DM-Lab-30 environments (that have inputs involving language instructions, pixel images, and rewards, among others) but it can also be effective at hard exploration tasks in DM-Hard-8 environments.
\end{abstract}
\section{Introduction}

Developing reinforcement learning (RL)  methods to tackle problems that have complex observations, partial observability, and sparse reward signals is an active research direction.
Such problems provide interesting challenges because the observations can often contain multi-modal information consisting of visual features, text, and voice. 
Coupled with partial observability, an agent must look at a sequence of past observations and learn useful representations to solve the RL task.
Additionally, when the reward signal is sparse, exploring exhaustively becomes impractical and careful consideration is required to determine which parts of the environment should be explored.
That is, efficient exploration depends on an agent's representation.
Therefore, a unified framework to tackle both representation learning and exploration can be beneficial for tackling challenging problems. 

Approaches for representation learning often include auxiliary tasks for predicting future events, either directly \citep{jaderberg2016reinforcement} or in the latent space  \citep{guo2020bootstrap,guo2022byol,schwarzer2020data}. 
However, in the rich observation setting, fully reconstructing observations may not be practical, while reconstruction in the latent space is prone to representation collapse.
Alternatively, auxiliary tasks using contrastive losses can prevent representation collapse \citep{chen2020simple,srinivas2020curl}, but can be quite sensitive to the choice of negative sampler \citep{zbontar2021barlow}.

Similarly, approaches targetting exploration often include reward bonuses based on counts  \citep{bellemare2016unifying}, or reconstruction errors \citep{schmidhuber1991possibility,pathak2017curiosity}. While these ideas are appealing, in rich-observation settings enumerating counts can be challenging and reconstruction can focus on irrelevant details, resulting in misleading bonuses. Methods like RND \citep{burda2018exploration} partially address these concerns but are completely decoupled from any representation learning method. This substantially limits the opportunity for targeted exploration in the agent's model of the environment.    
A detailed overview of other related work is deferred to Section \ref{sec:related} and Appendix \ref{sec:apx:related}.

In this work we provide De-Rex: \underline{de} composition-based \underline{r}epresentation and \underline{ex}ploration, an approach that builds upon singular value decomposition (SVD) to train \textit{both} a useful state representation and a pseudo-count estimate for exploration in rich-observation/high-dimensional settings.

\textbf{Structure preserving representations: } De-Rex uses a decomposition procedure to learn a state representation that preserves the transition structure without ever needing to build the transition matrix (Section \ref{sec:tabular}). Further, our method requires only access to transition samples in a mini-batch form and is compatible with deep networks-based decomposition (Section \ref{sec:fa}).
This improves upon prior work which used decompositions but was restricted to tabular settings \citep{mahadevan2005proto,mahadevan2007proto} or required symmetric matrices (e.g., Laplacian  \cite{wu2018laplacian}) that could corrupt
the underlying structure when the transition kernel is asymmetric.

\textbf{Pseudo-counts for exploration: } We show that the norms of these learned state representations capture relative state visitation frequency, thereby providing a very cheap procedure to obtain pseudo-counts (Sections \ref{sec:tabular} and \ref{sec:fa}). This significantly streamlines the exploration procedure in contrast to prior works that required learning exploratory options using spectral decomposition of the transition kernel \citep{machado2017laplacian,machado2017eigenoption}.

\textbf{Performance on large-scale experiments: } Building upon ideas from predictive state representations \citep{littman2001predictive,singh2003learning}, De-Rex also provides a procedure to tackle representation learning and exploration in partially observable environments (Section \ref{sec:PSR}).
 We demonstrate the effectiveness of De-Rex on DM-Lab-30 \citep{beattie2016deepmind} and  DM-Hard-8 exploration tasks \citep{paine2019making}, in a multi-task setting (\textit{without} access to task label). These domains are procedurally generated, partially observable, and have multi-modal observations involving language instructions, pixel images, and rewards (Section \ref{sec:results}).

\section{Preliminaries}
We will begin in a Markov decision process (MDP) setting and develop the key insights using singular value decomposition (SVD).
Later in Section \ref{sec:PSR}, we will extend the ideas to  partially observable MDPs (POMDPs).

Any vector $x$ is treated as a column vector. Any matrix $\bf M$ is denoted using bold capital letters, and ${\bf M}_{[i,j]}$ correspond to $i^\text{th}$ row and $j^\text{th}$ column of $\bf M$.
An MDP is a tuple $(\mathcal S, \mathcal A, p, r, d_0)$, where $\mathcal S$ is the set of finite states, $\mathcal A$ is a set of finite actions, $r$ is the reward function, $p$ is the transition function, and $d_0$ is the starting-state distribution. 
A policy $\pi$ is a decision making rule, and let ${\bf P_\pi} \in \mathbb R^{|\mathcal S| \times |\mathcal S|}$ be the transition matrix induced by $\pi$ such that ${\bf P_\pi}_{[s,s']} \coloneqq \sum_{a \in \mathcal A}\Pr(s'|s,a)\pi(a|s)$.
Let the performance of a policy $\pi$ be $J(\pi) \coloneqq \mathbb E_\pi\left[\sum_{t=0}^\infty \gamma^t R_t\right]$, where $\gamma \in [0,1)$ and $R_t$ is the reward observed at time $t$ on interacting with the environment with $\pi$.
Let $\Pi$ be a class of (parameterized) policies. The goal is to find $\argmax_{\pi \in \Pi} J(\pi)$.


\subsection{Why SVD?}

\begin{figure}[t]
    \centering
    \includegraphics[width=0.2\textwidth]{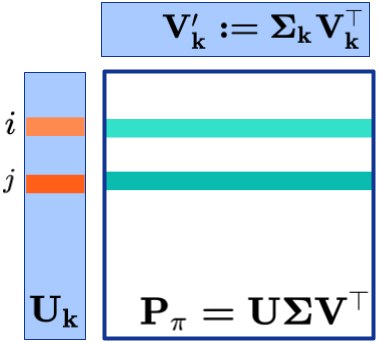}
    \caption{
    Consider top-k SVD decomposition of ${\bf P_\pi} \approx {\bf U_k V_k'}$, where ${\bf V_k'} \coloneqq {\bf \Sigma_k V_k}^\top \in \mathbb R^{k \times |\mathcal S|}$. After decomposition, we consider the rows of ${\bf U_k}\in \mathbb R^{|\mathcal S|\times k}$ to be the $k \ll |\mathcal S|$ dimensional representation for the states. Intuitively, if two states have similar next state transition distributions (\textcolor{teal}{teal} rows ${\bf P_{\pi}}_{[i]}$ and ${\bf P_{\pi}}_{[j]}$) then     their representations (\textcolor{orange}{orange} rows ${\bf U_{k}}_{[i]}$ and ${\bf U_{k}}_{[j]}$) would also be similar.
    Notice that rows in $\bf U_k$ can be thought of as coefficients for combining the common basis vectors $\bf V_k'$ to best estimate rows of $\bf P_{\pi}$   (i.e., ${\bf P_\pi}_{[i]} \approx {\bf U_{k}}_{[i]} {\bf V_k'}$,
    and  ${\bf P_{\pi}}_{[j]} \approx {\bf U_{k}}_{[j]} {\bf V_k'}$). 
    Therefore, if ${\bf P_{\pi}}_{[i]} \approx {\bf P_{\pi}}_{[j]}$,
    then ${\bf U_{k}}_{[i]} {\bf V_k'} \approx {\bf U_{k}}_{[j]} \bf V_k'$.
    }
    \label{fig:svd}
\end{figure}
Following the seminal work on proto-value functions \citep{mahadevan2005proto,mahadevan2007proto},  representations obtained from spectral decomposition  have  been shown to be effective not only towards minimizing the Bellman error \citep{behzadian2019fast} but also stabilizing off-policy learning \citep{ghosh2020representations}.
Of particular interest are recent works by \citet{lyle2021effect,lyle2022learning} that show under idealized conditions that\textit{ state features }emerging from value function estimation, using just TD-learning, capture the top-k subspace of the eigenbasis of ${\bf P_\pi}$.
As we show in the following, these top-k eigenvectors 
of ${\bf P_\pi}$ are \textit{exactly} the same as the top-k eigenvectors of $({\bf I} - \gamma {\bf P_\pi})^{-1}$.
Therefore, these basis vectors can be  particularly useful for obtaining state representations as the value function is linear in $({\bf I} - \gamma {\bf P_\pi})^{-1}$  \citep{sutton2018reinforcement}. 
That is, $v^\pi = ({\bf I} - \gamma {\bf P_\pi})^{-1}R_\pi$, where $R_\pi \in \mathbb 
R^{ |\mathcal S| }$ is a vector with one-step expected reward ${\bf E}_\pi[R|s] $ for all the states $s \in \mathcal S$.
%

\begin{thm}
\thlabel{thm:ebf}
    If the eigenvalues of ${\bf P_\pi} \in \mathbb R^{|\mathcal S|\times |\mathcal S|}$
    are real and distinct,
    \footnote{we consider eigenvalues to be real and distinct to simplify the choice of `top-k' eigenvalues.} then for any $k$, the top-k eigenvectors of  ${\bf P_\pi}$ and $({\bf I} - \gamma {\bf P_\pi})^{-1}$ are the same. 
\end{thm}
Proofs for all the results are deferred to  Appendix \ref{apx:unbconsproof}.
These lead to a natural question: \textit{Can we explicitly decompose ${\bf P_\pi}$ to aid representation learning?} 
To answer this, \thref{thm:ebf} can be useful from a practical standpoint, as estimating and decomposing ${\bf P_\pi}$ can be much easier than doing so for $({\bf I} - \gamma {\bf P_\pi})^{-1}$. 
Unfortunately, however, eigen decomposition of ${\bf P_\pi}$ need not exist in general when ${\bf P_\pi}$ is not diagonalizable, for example, if ${\bf P_\pi}$ is not symmetric. 

Therefore, we instead consider singular value decomposition (SVD), which generalizes eigen decomposition, of ${\bf P_\pi}$ and always exists \citep{van1976generalizing}.
That is, there always exists orthogonal matrices
${\bf U} \in \mathbb R^{|\mathcal S|\times |\mathcal S|}$, ${\bf V}\in \mathbb R^{|\mathcal S|\times |\mathcal S|}$, and a real-valued diagonal matrix ${\bf \Sigma} \in \mathbb R^{|\mathcal S|\times |\mathcal S|}$ such that  ${\bf P_\pi} = {\bf U\Sigma V^\top}$.
In Figure \ref{fig:svd}, we provide some intuition for what the top-k SVD decomposition of ${\bf P_\pi}$ captures: ${\bf P_\pi} \approx {\bf U_k}{\bf \Sigma_k V_k}^\top$, where ${\bf U_k}\in \mathbb R^{|\mathcal S|\times k}$ and ${\bf V_k}\in \mathbb R^{|\mathcal S|\times k}$ are the top-k left and right singular vectors corresponding to the top-k singular values ${\bf \Sigma_k}\in \mathbb R^{k\times k}$.  
We build upon this insight to develop an SVD-based scalable method that goes beyond just representation learning and can also be useful to obtain intrinsic bonuses to drive exploration.

Outside of RL, different matrix decomposition methods such as Kernel PCA  \citep{scholkopf1997kernel}, deep decomposition \citep{de2021survey,De_Handschutter_2021}, and EigenGame  \citep{gemp2020eigengame,deng2022neuralef} have also been proposed.
 In what follows, we draw inspiration from these but focus on the challenges for the RL setup, i.e., asymmetric (transition) matrices, construction of exploration bonuses, and partial observability.

\section{Tabular Setting}
\label{sec:tabular}
\begin{figure*}[t]
    \centering
    \includegraphics[width=0.23\textwidth]{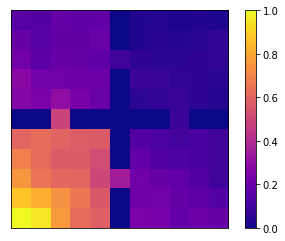}
    \includegraphics[width=0.2\textwidth]{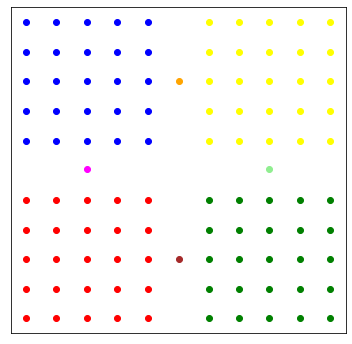}
    \includegraphics[width=0.21\textwidth]{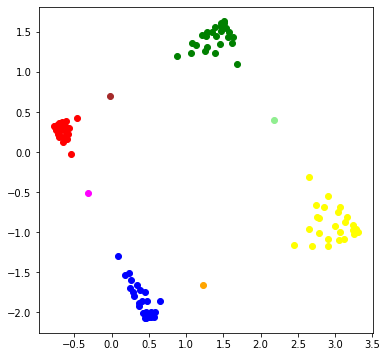}
    \includegraphics[width=0.23\textwidth]{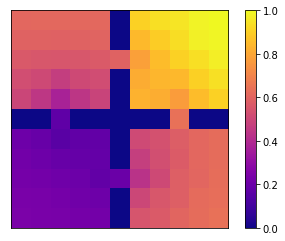}
    \caption{\textbf{(A)} State visitation frequency in the data collected using a random policy. \textbf{(B)} Reference colors for the states in the 4rooms. \textbf{(C)} 2D Representations learned from 121-dimension tabular representation. \textbf{(D)} Bonus constructed using the learned representations.  Legend: values are normalized between $0$ and $1$ and a darker color represents a lower value. } 
    \label{fig:tabular4room}
\end{figure*}

To convey the core insights and lay the foundations for our approach, we will first consider the tabular setting.
In Section \ref{sec:fa}, we will extend these insights to develop scalable methods that use function approximations.

Decomposing ${\bf P_\pi}$ exactly would require transitions corresponding to all the states.
This is clearly impractical, as an agent would rarely have access to all the possible states.
Therefore, we consider decomposing a weighted version of ${\bf P_\pi}$.
Let $\mathcal D \coloneqq \{(s,a,r,s')_i \}_{i=1}^n$ be a dataset of $n$ transition tuples collected using $\pi$.
Let $d \in \Delta(\mathcal S)$ be the distribution of the states $s$ from which the observed transition tuples are sampled, and let ${\bf D} \coloneqq \texttt{diag}(d) \in \mathbb R^{|\mathcal S| \times |\mathcal S|}$.
In the following, we will consider decomposing the weighted matrix ${\bf D}{\bf P_\pi}$ (when $d$ is uniform, ${\bf D}{\bf P_\pi} \propto {\bf P_\pi}$).
As we will elaborate in the following sections, this choice of decomposing $\bf D P_\pi$ has two major advantages: \textbf{(a)} It will provide a practical representation learning procedure that only requires sampling observed transition tuples, and \textbf{(b)} it will also provide, almost automatically, estimates of pseudo-counts that will be useful for exploration.

\textbf{Representation Learning: } 
In this section, we outline the key idea for representation learning.  
Before proceeding, we introduce some additional notation.
 Let $x_{s} \in \mathbb R^{|\mathcal S|}$ be the one-hot encoding corresponding to state $s$, and let ${\bf X} \in \mathbb R^{n\times |\mathcal S|}$ be the stack of $x_{s}$ for the states $s$ in the dataset $\mathcal D$. Similarly, let $y_{s} \in \mathbb R^{|\mathcal S|}$ be the one-hot encoding corresponding to $s'$ and let ${\bf Y} \in \mathbb R^{n\times |\mathcal S|} $ be the matrix consisting the stack of $s'$ in the dataset $\mathcal D$.

Using the dataset $\mathcal D$ we provide a sample estimate for ${\bf D}{\bf P_\pi}$ that we would like to decompose.
While we will \textit{never} actually construct this, a possibly intractable, estimate of ${\bf D}{\bf P_\pi}$, it will be useful for deriving the proposed procedure.
Thus, let
\begin{align}
    {\bf A_n} \coloneqq \frac{1}{n} {\bf X}^\top {\bf Y} = \frac{1}{n}\sum_{i=1}^n x_i y_i ^\top. \label{eq:xy}
\end{align}
\begin{thm}
\thlabel{thm:unbcons}
    ${\bf A_n}$ is an unbiased and a consistent estimator of ${\bf D}{\bf P_\pi}$, i.e.,
    \begin{align}
       \forall i,j, \quad \mathbb E\left[{\bf A_n}_{[i,j]}\right] = {\bf D}{\bf P_\pi}_{[i,j]}, && {\bf A_n}_{[i,j]} \overset{\text{a.s.}}{\longrightarrow} {\bf D}{\bf P_\pi}_{[i,j]}.
    \end{align}
\end{thm}
With a slight overload of notation, let ${\bf U}\in \mathbb R^{|\mathcal S|\times |S|}$ and ${\bf V}\in \mathbb R^{|\mathcal S|\times |S|}$ be the left and the right singular vectors for ${\bf A_n}$, and let ${\bf\Sigma} \in \mathbb R^{|S|\times |S|}$ be the singular values such that 
\begin{align}
    {\bf A_n} &= {\bf U\Sigma V}^\top. \label{eq:svd1}
\end{align}
Now pre-multiplying both sides in \eqref{eq:svd1} with $\bf U^\top$, post-multiplying with $\bf V$ and observing that $\bf U^\top  U= V^\top V= I$,
\begin{align}
   {\bf U}^\top {\bf A_n} {\bf V}&= {\bf \Sigma}. \label{eq:svd2}
\end{align}
Expanding \eqref{eq:svd2} using \eqref{eq:xy}, we obtain
\begin{align}
    \frac{1}{n} \sum_{i=1}^n {\bf U^\top} x_i y_i^\top {\bf V}&= {\bf\Sigma}. \label{eq:svd3}
\end{align}
Now, let $f_{\bf U}: \mathcal S \rightarrow \mathbb R^{|S|}$ and $g_{\bf V}: \mathcal S \rightarrow \mathbb R^{|S|}$ be linear functions parameterized by `weights' $\bf U$ and $\bf V$ respectively. 
Then \eqref{eq:svd3} can be expressed as,
\begin{align}
        \frac{1}{n} \sum_{i=1}^n f_{\bf U}(x_i) g_{\bf V}(y_i)^\top &= \bf \Sigma. \label{eqn:linearcov}
\end{align}
The form in \eqref{eqn:linearcov} is useful because it naturally leads to a loss function to search for `parameters' $\bf U$ and $\bf V$ that provide the decomposition in \eqref{eq:svd1}.
Specifically, if we want the top-k decomposition, let $\widehat {\bf U}\in \mathbb R^{|S|\times k}$ and $\widehat {\bf V}\in \mathbb R^{|S|\times k}$ be the estimates for top-k left and right singular vectors, and let ${\bf \Sigma(\widehat U, \widehat V)} \in \mathbb R^{k \times k}$ be defined such that,
\begin{align}
        {\bf \Sigma \left(\widehat U, \widehat V\right)} \coloneqq \frac{1}{n} \sum_{i=1}^n f_{\bf \widehat U}(x_i) g_{\bf \widehat V}(y_i)^\top, \label{eqn:C}
\end{align}
where $f_{\bf \widehat U}: \mathcal S \rightarrow \mathbb R^{k}$ and $g_{\bf \widehat V}: \mathcal S \rightarrow \mathbb R^{k}$.
To perform the top-k decomposition, one could now search for $\bf \widehat U$ and $\bf\widehat V$ that makes $\bf \Sigma(\widehat U, \widehat V)$ resemble the properties of $\bf \Sigma$ i.e., $\bf \Sigma(\widehat U, \widehat V)$ is a diagonal matrix with large values on the diagonal.
Specifically, we consider the following optimization problem where given ${\bf \widehat U ^\top \widehat U} = {\bf \widehat V^\top \widehat V}  = \bf I$ the diagonal elements of $\bf \Sigma(\widehat U, \widehat V)$ are maximized while the off-diagonal elements are $0$,
\begin{align}
   {\bf U^*, V^*} &\coloneqq \argmin_{\bf \widehat U, \widehat V}
   - \sum_{i=1}^k {\bf \Sigma (\widehat U, \widehat V)}_{[i,i]} 
   \\
   & \hspace{60pt} \text{s.t.}  \,\,   {\bf \Sigma (\widehat U, \widehat V )}_{[i,j]} = 0 \quad  \forall i\neq j.
   \label{eqn:linearSVD}
\end{align}
Resulting function $f_{\bf U^*}$ provides the representation for the states. For instance,  similar to the example in Figure \ref{fig:svd}, $f_{\bf U^*}(x_s) =  {\bf U^*}^\top x_s$, which is the row of ${\bf U^*} \in \mathbb R^{|S|\times k}$ corresponding to the state $s$ because $x_s$ is one-hot. 
The function $g_{\bf V^*}$ is only used in searching for the function $f_{\bf U^*}$.
In Figure \ref{fig:tabular4room}(c), we provide an illustrative example on the 4-rooms domains, where it can be observed that the learned representation clusters the states that have similar transition distributions (even when $\bf D$ is far from uniform).

\paragraph{Exploration:}   
We now discuss how representations obtained from the SVD decomposition of ${\bf D}{\bf P_\pi}$ can also be beneficial for exploration.
In the following, we show that a pseudo-count estimate of a state's visitation frequency, relative to other states, can be obtained instantly by just computing the norms of the learned state representation. 

Let the decomposition of ${\bf D}{\bf P_\pi} = \bf U\Sigma V^\top$ and as earlier let the representation function $f_{\bf U}$ be parameterized linearly with left singular values $\bf U$. 
Let ${\bf \Lambda \coloneqq \Sigma}^2$, then the weighted norm of the representations $\lVert f_{\bf U}(x_s)\rVert_{{\bf \Lambda}^{-1}}$ is inversely proportional to $d(s)$, i.e.,\textit{ how often $s$ has been observed in the dataset relative to other states}.
\begin{thm}
\thlabel{thm:exp}
    If ${\bf P_\pi} {\bf P_\pi} ^\top$ and $\bf D$ are invertible,
    then for $\alpha_{s} \coloneqq ({\bf P_\pi} {\bf P_\pi}^\top)^{-1}_{[s,s]}$,
    \begin{align}
        \big\lVert f_{\bf U}(x_s)\big\rVert _{\bf \Lambda^{-1}}^2 = \frac{\alpha_s}{d(s)^2}. \label{eqn:bonus}
    \end{align}
\end{thm}
Here $\alpha_s$ is a state-dependent constant that is independent of $d(s)$.
\thref{thm:exp} is particularly appealing because not only can $f_U$ provide representations that respect the similarity in terms of next-state transitions, but it also orients the representations such that norms of the representations can be used to indicate pseudo-counts.
Further estimating the bonus only requires computing $\big\lVert f_{\bf U}(x_s)\big\rVert _{\bf \Lambda^{-1}}^2$, where $\bf \Lambda$ is a diagonal matrix which can be inverted in $O(k)$ time, where $k$ is the dimension of the representation. Therefore computing the norm is also an $O(k)$ time operation.
Figure \ref{fig:tabular4room}(D), provides an illustrative example  of what $\lVert f_{\bf U}(x_s)\rVert_{{\bf \Lambda}^{-1}}$ captures in the 4-rooms domain.

\begin{rem}
Intuitively, there are extra degrees of freedom in how representations can be rotated/re-oriented such that their relative distances from each other are the same, but their norm changes.  This freedom is leveraged in \eqref{eqn:bonus} to re-orient the representations such that states that have higher visitation tend to get representations that have a lower norm value. This can be visualized in Fig \ref{fig:tabular4room}(C) and \ref{fig:4roomfa}(C).
\end{rem}

\begin{rem}
Alternatively, observe that $\bf D \bf P_\pi$ is related to the (normalized)  \textit{count matrix} \citep{bellemare2016unifying}. Therefore, decomposing it results in representations that have properties of both the visitation frequency $\bf D$ and the transition structure $\bf P_\pi$.
\end{rem}

\section{Function Approximations}
\label{sec:fa}
\begin{figure*}[t]
    \centering
    \includegraphics[width=0.23\textwidth]{images/heatmap.png}
    \includegraphics[width=0.2\textwidth]{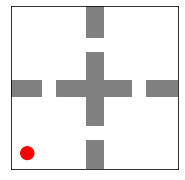}
    \includegraphics[width=0.2\textwidth]{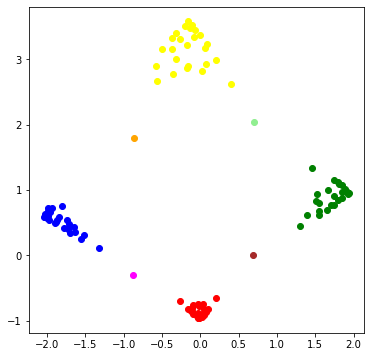}
    \includegraphics[width=0.23\textwidth]{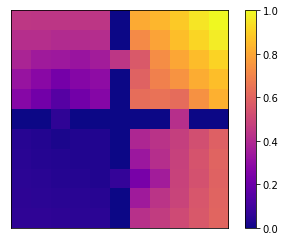}
      \caption{\textbf{(A)}  State visitation frequency in the data collected using a random policy. \textbf{(B)} Example of the domain image used as the state. \textbf{(C)} 2D representations learned from images (30x30x3) corresponding to the states. Reference colors for the states are the same as that used in Figure \ref{fig:tabular4room} (B). \textbf{(D)} Bonus constructed using the learned representations. Legend: values are normalized between $0$ and $1$ and a darker color represents a lower value. 
      }
    \label{fig:4roomfa}
\end{figure*}

While \eqref{eqn:linearSVD} provides the problem \textit{formulation}, it does not provide a scalable way to \textit{perform} the optimization.
Further, we would want to extend the procedure to work with deep neural networks.
In the following, we elaborate on a way to achieve that.

\paragraph{Representation Learning}
Leveraging the form in \eqref{eqn:linearcov}, we extend the idea using non-linear functions $f_{\theta_1}:  \mathbb R^d \rightarrow \mathbb R^k$ and $g_{\theta_2}:  \mathbb R^d \rightarrow \mathbb R^k$ that are parameterized by weights $\theta \coloneqq [\theta_1, \theta_2]$ and take states (need not be tabular) as inputs,
\begin{align}
    {\bf \Sigma}(\theta) \coloneqq \frac{1}{n} \sum_{i=1}^n f_{\theta_1}(s_i) g_{\theta_2}(s'_i)^\top & \in \mathbb{R}^{k\times k}. \label{eqn:fullC}
\end{align}
Specifically, let
\begin{align}
    \mathcal L_{\operatorname{diag}}\left({\bf \Sigma}(\theta) \right) &\coloneqq \frac{1}{k} \sum_{i=1}^k {\bf \Sigma}(\theta)_{[i,i]}, \label{eqn:svdl1}
     \\
    \mathcal L_{\operatorname{off}}\left({\bf \Sigma}(\theta) \right) &\coloneqq \frac{1}{k^2-k} \sum_{\substack{i,j=1\\ i\neq j}}^k {\bf \Sigma}(\theta)_{[i,j]}^2. \label{eqn:svdl2}
\end{align}
Now similar to \eqref{eqn:linearSVD} we define a loss with soft constraints by combining \eqref{eqn:svdl1} and \eqref{eqn:svdl2},
\begin{align}
    \mathcal L({\bf \Sigma}(\theta)) &\coloneqq - \mathcal L_{\operatorname{diag}}({\bf \Sigma}(\theta)) + \lambda_r \mathcal L_{\operatorname{off}}({\bf \Sigma}(\theta)), \label{eqn:loss}
\end{align}
where $\lambda_r$ is a hyper-parameter.
For the tabular setting in \eqref{eqn:linearSVD}, the condition for $\bf \widehat U^\top \widehat {\bf U}= \widehat {\bf V}^\top \widehat {\bf V}= I$ prevents unbounded maximization of the diagonal by ensuring that the columns of  $\bf U$ (i.e., each feature dimension), and columns of $\bf V$ have a norm of one.
In the function approximation setting it is not immediate what this condition should be. Therefore, to prevent unbounded maximization of ${\bf  \Sigma(\theta)}_{[i,i]}$ in \eqref{eqn:svdl1}, we use batchnorm \citep{ioffe2015batch} on the output of both $f$ and $g$ to ensure that each feature dimension remains normalized.
In Figure \ref{fig:4roomfa}(C) we illustrate the representations learned by minimizing the loss in \eqref{eqn:loss} with convolutional neural networks, batch-normalization, and $\lambda_r=1$.

\paragraph{Mini-batch optimization: }
In practice, computing ${\bf \Sigma}(\theta)$ in \eqref{eqn:fullC} using all $n$ samples is impractical. Instead, it would be preferable to have a mini-batch estimate of ${\bf \Sigma}(\theta)$,
\begin{align}
    \widehat {\bf \Sigma}(\theta) \coloneqq \frac{1}{b}\sum_{i=1}^b f_{\theta_1}(s_i)  g_{\theta_2}(s'_i)^\top, \label{eqn:batchC}
\end{align}
where $b \ll n$ is the mini-batch size.
Unfortunately, it can be shown that while the gradients for $\mathcal L_{\operatorname{diag}}({\bf \widehat \Sigma}(\theta))$ are unbiased,  gradients for $\mathcal L_{\operatorname{off}}({\bf \widehat \Sigma}(\theta))$ will be biased in general. 
\begin{thm}
\thlabel{thm:bias}
Gradient of $\mathcal L_{\operatorname{diag}}({\bf \widehat \Sigma}(\theta))$ is an unbiased estimator of  $\mathcal L_{\operatorname{diag}}({\bf  \Sigma}(\theta))$, and  gradient of $\mathcal L_{\operatorname{off}}({\bf \widehat \Sigma}(\theta))$ is in general a biased estimate of $\mathcal L_{\operatorname{off}}({\bf \Sigma}(\theta))$ , i.e.,
    \begin{align}
    \mathbb E\left[\partial_\theta \mathcal L_{\operatorname{diag}}\left(\widehat {\bf \Sigma}(\theta) \right)\right] &= 
    \partial_\theta \mathcal L_{\operatorname{diag}}({\bf \Sigma}(\theta)), 
    \\
    \mathbb E\left[\partial_\theta \mathcal L_{\operatorname{off}}\left(\widehat {\bf \Sigma}(\theta) \right)\right] &{\,\, \color{red}\neq} \,\,
    \partial_\theta \mathcal L_{\operatorname{off}}({\bf \Sigma}(\theta)), 
\end{align}
where the expectation is over the randomness of the mini-batch.
\end{thm}
This problem can be attributed to the fact that in \eqref{eqn:svdl2},
\begin{align}
    {\bf \Sigma}(\theta)_{[i,j]}^2 = \left(\frac{1}{n} \sum_{k=1}^n f_\theta(s_k) \, g_\theta(s'_k)^\top_{[i,j]} \right)^2,
\end{align}
the square is \textit{outside} the expectation.
This is reminiscent of the issue in residual gradient optimization of the mean squared Bellman error \citep{baird1995residual}.
Therefore, to address this issue, we also use the double-sampling trick to obtain unbiased gradient estimates.
Specifically, let 
\begin{align}
     \widehat {\bf \widehat \Sigma}(\theta) \coloneqq \left(\frac{1}{b_2}\sum_{i=1}^{b_2} f_{\theta_1}(s_i)  g_{\theta_2}(s'_i)^\top\right), \label{eqn:batchC2}
\end{align}
be estimated using $b_2$ samples of a separate batch of data than the one used in \eqref{eqn:batchC}.
Then using \eqref{eqn:batchC} and \eqref{eqn:batchC2}, we use $\widehat {\mathcal L}_\text{off}(\theta)$ instead of ${\mathcal L}_\text{off}({\bf \Sigma}(\theta))$ in \eqref{eqn:svdl2}, where
\begin{align}
    \widehat {\mathcal L}_{\operatorname{off}}(\theta) &\coloneqq \frac{1}{(k^2-k)} \sum_{\substack{i,j=1\\ i\neq j}}^k \widehat {\bf \Sigma}(\theta)_{[i,j]} \,\, \operatorname{sg}\left(\widehat{\bf \widehat \Sigma}(\theta)_{[i,j]}\right) \\
    & \,\, +\frac{1}{(k^2-k)} \sum_{\substack{i,j=1\\ i\neq j}}^k \widehat{\bf \widehat \Sigma}(\theta)_{[i,j]} \,\, \operatorname{sg}\left(\widehat{\bf \Sigma}(\theta)_{[i,j]}\right),
    \label{eqn:svdl3}
\end{align}
where $\operatorname{sg}$ corresponds to the stop-gradient operator.
\begin{thm}
\thlabel{thm:bias2}
Gradient $\partial_\theta \widehat{ \mathcal L}_{\operatorname{off}}(\theta)$ is an unbiased estimator of $\partial_\theta \mathcal L_{\operatorname{off}}({\bf \Sigma}(\theta))$, i.e.,
    $$
    \mathbb E\left[\partial_\theta \widehat{ \mathcal L}_{\operatorname{off}}(\theta) \right]  = 
    \partial_\theta \mathcal L_{\operatorname{off}}({\bf \Sigma}(\theta)). $$
\end{thm}
\paragraph{Exploration: }
Having developed a mini-batch optimization procedure to obtain the representations, we now extend \eqref{eqn:bonus} to the function approximation setting.
Specifically, using
${\bf \Lambda}(\theta) \coloneqq {\bf \Sigma}^2(\theta)$, we construct the estimate for pseudo-count using $\lVert f_{\theta_1}(s)\rVert_{{\bf \Lambda}(\theta)^{-1}}$.

However, as discussed earlier, obtaining ${\bf \Sigma}(\theta)$ exactly can be hard in practice. Therefore, in practice \textbf{(a)} we keep a running mean estimate $\widetilde {\bf \Sigma}(\theta)$ for different $\widehat {\bf \Sigma}(\theta)$ observed across mini-batches.
Further, \textbf{(b)} as the loss in \eqref{eqn:loss} encourages $\widetilde {\bf \Sigma}(\theta)$  to be diagonal, we only maintain a running estimate $\widetilde {\bf \Sigma}(\theta)$ for the diagonal elements and keep the off-diagonal elements as zero.
Now using
$\widetilde {\bf \Lambda}(\theta) \coloneqq  \widetilde {\bf \Sigma}^2(\theta)$, we construct the estimate for pseudo-counts of $s$ using $\lVert f_{\theta_1}(s)\rVert_{\widetilde {\bf\Lambda}(\theta)^{-1}}$.
In Figure \ref{fig:4roomfa} (D) we illustrate the effectivity of estimating pseudo-counts using this procedure.

\section{POMDP \& Predictive State Representation}
\label{sec:PSR}
An important advantage of the proposed approach is that we can consider other matrices (e.g., ${\bf P} \in \mathbb R^{|S||A|\times |S|}$ instead of ${\bf P_\pi} \in \mathbb R^{|S|\times |S|}$) as well. 
In this section, we take a step further and consider the\textit{ system-dynamics matrix} that are well-suited for \textit{partially observable} settings.

\paragraph{Brief Review: }
For a partially observable system, let $\mathcal O$ be the set of possible observations and let $O_t \in \mathcal O$ be a random variable for the observations at time $t$.
Let $\mathcal H$ be a set of observation-action sequences, such that $\mathcal H = \bigcup_{t\in \{0,1, ...,T\}} \mathcal H_t$, where $T$ is the horizon length, $\mathcal H_0= \mathcal O$, and $\forall t>0, \mathcal H_t = \mathcal H_{t-1} \times \mathcal A \times \mathcal O$.
Let a \textit{test} $\tau \in \mathcal H$ be any sequence of action and observations $(a_1, o_1,  ...., a_j, o_j)$. 
Given a history $h=(o_0,a_0,...,a_{i-1},o_i) \in \mathcal H$, let
$p(\tau|h) \coloneqq  \Pr(O_{i+1}=o_1, ..., O_{i+j}=o_j|h, A_i=a_1, ..., A_{i+j-1}=a_j)$ be the history-conditioned prediction of a test $\tau$.
Intuitively, $p(\tau|h)$ encodes the probability of future outcomes given the observed history.
A system-dynamics matrix ${\bf W} \in \mathbb R^{|\mathcal H|\times|\mathcal H|}$ is defined to have its rows correspond to histories and columns correspond to tests, such that $ {\bf W}_{[i,j]} \coloneqq p(\tau_j|h_i)$.
Figure \ref{fig:PSR} provides an illustration of a low-rank system-dynamics matrix ${\bf W}={\bf UV}'$, where ${\bf U}\in \mathbb R^{|\mathcal H|\times k}$, ${\bf V}' \in \mathbb R^{k \times |\mathcal H|}$. Further, let $(z_i)_{i=1}^k$ correspond to the latent variables of the system.
\begin{figure}[h]
    \centering
    \includegraphics[width=0.49\textwidth]{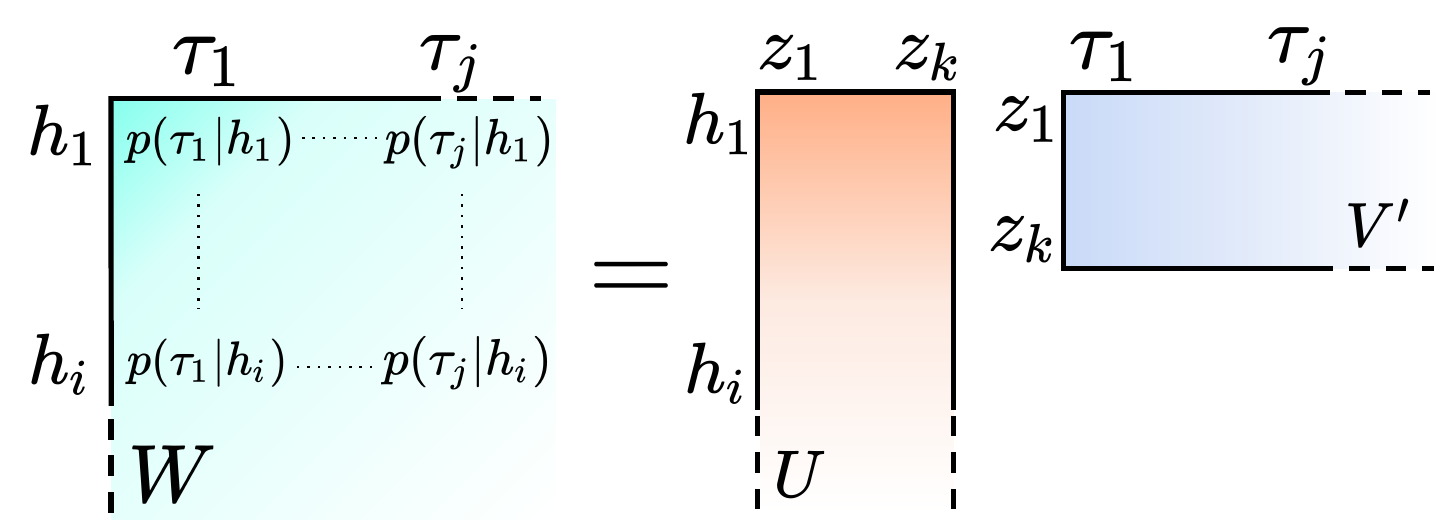}
    \caption{A low-rank system-dynamics matrix (adapted from the work by \citet{singh2012predictive}). Intuitively, row $\bf U_{[i]}$ provides a good representation for history $h_i$ because it contains \textit{all the information that is sufficient to infer the outcome probability of all the futures (tests), given $h_i$}.}
    \label{fig:PSR}
\end{figure}

Viewing $z_i$'s from different perspectives provides different insights.
If a POMDP has $k$ unobserved states then $z$ can correspond to the hidden states of the system. Correspondingly, the rows ${\bf U}_{[i]}$ would be the belief distribution over those $k$-hidden states, given the history $h_i$ so far, and columns of ${\bf V}'$ contain the probability of observing the futures/tests if the agent is in a given hidden state.
In contrast,  PSRs consider $z$ to be \textit{core tests} \citep{littman2001predictive,singh2003learning,singh2012predictive, james2004learning,rudary2003nonlinear}, such that the probability of all the remaining tests can be obtained as linear combinations of the probabilities of these core tests occurring.
Therefore, ${\bf U}_{[i]}$ encodes the probability of these $k$ core tests occurring given the history $h_i$, and the columns of ${\bf V}'$ contain the coefficients for the \textit{linear combination}.
Mathematically, this subtle difference from the standard POMDP perspective allows $\bf V'$ to have values that are not constrained to be probability distribution nor be positive, and thus allows PSR to model a much larger class of dynamical systems \citep{singh2012predictive}.
Transformed PSR (TPSR) extends this perspective further and lets ${\bf U}_{[i]}$ correspond to be a small number of \textit{linear combinations} of the probabilities of a larger number of tests.
Mathematically, the key difference is that this permits ${\bf U}_{[i]}$ to not be probability distributions and have negative values as well.
This perspective is particularly useful as the left singular vectors from the SVD on $\bf W$ now directly provide sufficient information from histories needed to infer future outcomes \citep{rosencrantz2004learning}.
\begin{figure*}[h]
    \centering
    \includegraphics[width=0.1\textwidth]{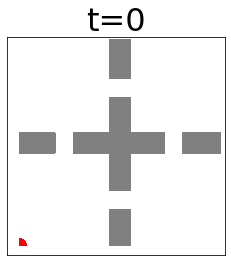}
    \includegraphics[width=0.1\textwidth]{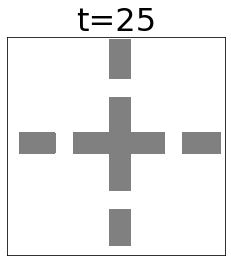}
    \includegraphics[width=0.1\textwidth]{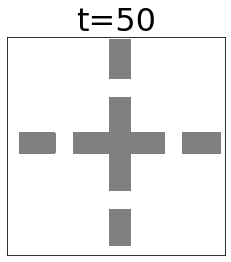}
    \includegraphics[width=0.1\textwidth]{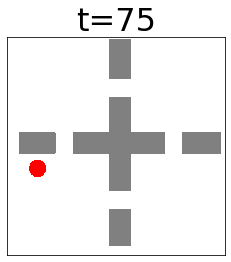}
    \includegraphics[width=0.1\textwidth]{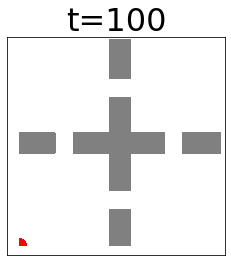}
    \includegraphics[width=0.1\textwidth]{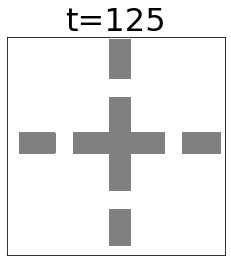}
    \includegraphics[width=0.1\textwidth]{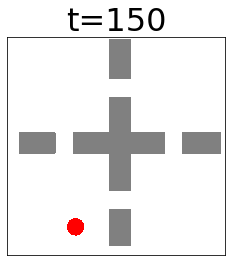}
    \includegraphics[width=0.1\textwidth]{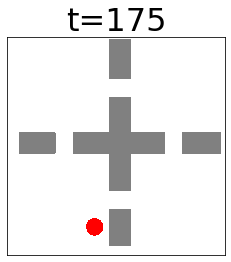}
    \includegraphics[width=0.1\textwidth]{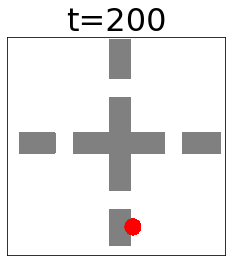}
    \\
    \includegraphics[width=0.1\textwidth]{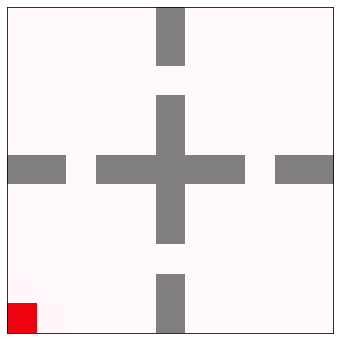}
    \includegraphics[width=0.1\textwidth]{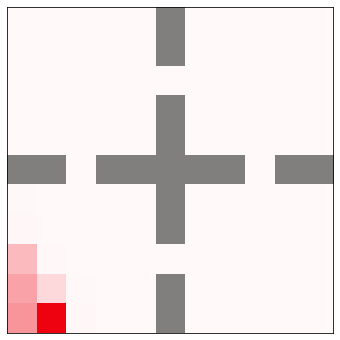}
    \includegraphics[width=0.1\textwidth]{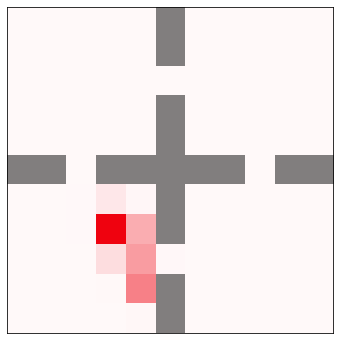}
    \includegraphics[width=0.1\textwidth]{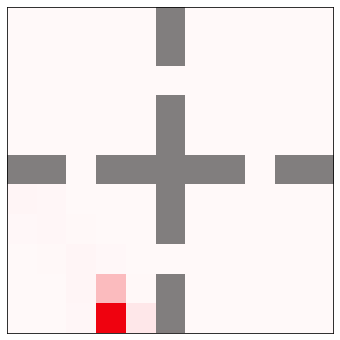}
    \includegraphics[width=0.1\textwidth]{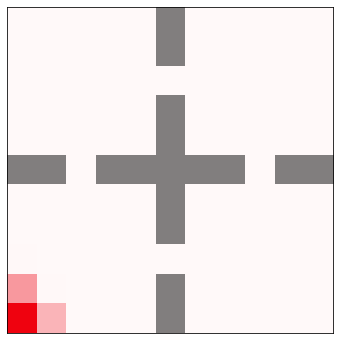}
    \includegraphics[width=0.1\textwidth]{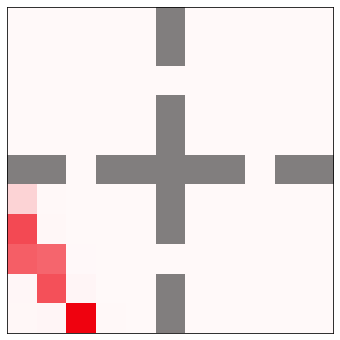}
    \includegraphics[width=0.1\textwidth]{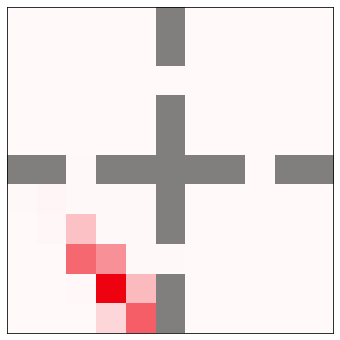}
    \includegraphics[width=0.1\textwidth]{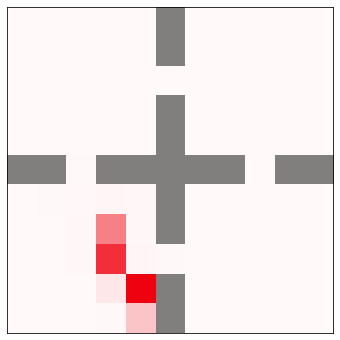}
    \includegraphics[width=0.1\textwidth]{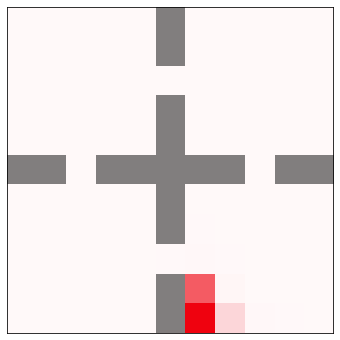}
    \caption{
Similar to Figure \ref{fig:tabular4room} and \ref{fig:4roomfa}, we provide plots in Figure \ref{fig:POMDP_rep} (in Appendix \ref{apx:ablationPOMDP}) to assess the quality of the learned distributions. However, it is not particularly informative for representations in higher dimensions.
Therefore, as an alternative, we learned a classification model (using a single-layer neural network) that aims to classify the true underlying state at the end of the observation sequence of a given (partial) history using the representation of that (partial) history. 
This provides us with a proxy for how well can the learned representations be used as to decode beliefs over the underlying true state.
Top-row illustrates the snapshots of the observations that were provided to the agent at timesteps $t=\{0,25,50,75,100,125,150,175,200\}$ (notice the agent missing in some frames). The bottom rows denote the belief decoded from the representation of the history till timestep $t$. See text for more details.  \textbf{Legend:} For the bottom row, the deeper red color indicates a higher degree of belief. For the top row, the red color indicates the location of the agent. }
    \label{fig:POMDP_toy2}
\end{figure*}

\subsection{De-Rex: Decomposition based Representation and Exploration }

We  build upon the idea of TPSR that was often restricted to small-scale setups \citep{rosencrantz2004learning, boots2011closing,boots2011online}, and extend them to a large-scale setting with function approximators using ideas developed in Section \ref{sec:fa}.
That is, in \eqref{eqn:fullC} we use $f_{\theta_1}: \mathcal H \rightarrow \mathbb R^k$  and $g_{\theta_2}: \mathcal H \rightarrow \mathbb R^k$ which take \textit{histories} and \textit{tests} as inputs, instead of states and next states, respectively.

Specifically, let
dataset $\mathcal D$ consist of $n$ trajectories.
For the $i^\text{th}$ trajectory $(o_0, a_0, ... a_{T-1}, o_{T})$ in $\mathcal D$, we define $h_{ij} \coloneqq (o_0, a_0,...,o_j)$ and $\tau_{ij} \coloneqq (a_j, o_{j+1}, ..., a_{T-1}, o_T)$.
We extend \eqref{eqn:fullC}, such that
\begin{align}
    {\bf \Sigma}(\theta) \coloneqq \frac{1}{nT} \sum_{i=1}^{n}\sum_{j=0}^{T-1} f_{\theta_1}(h_{ij}) g_{\theta_2}(\tau_{ij})^\top & \in \mathbb{R}^{k\times k}. \label{eqn:fullPSR}
\end{align}
Now, using the double sampling trick from  \eqref{eqn:svdl3} and using a loss similar to \eqref{eqn:loss}, we can diagonalize ${\bf \Sigma}(\theta)$ in \eqref{eqn:fullPSR}.
We provide more discussion on these in Appendix \ref{apx:implement}.

\textbf{4 rooms (POMDP): }Similar to Section \ref{sec:tabular} and \ref{sec:fa}, we also visualize the learned \textit{representations for the histories} and the estimated \textit{pseudo-counts for histories}.
For this purpose, we construct a domain where the observation $O_t$ at time $t$ corresponds to an  image of size $(30, 30, 3)$ that has the top-down view of the domain (similar to Figure \ref{fig:4roomfa} (B)), albeit that with probability $p$ the agent's marker (red dot) is completely hidden.
Therefore, when the marker is hidden, the observation $O_t$ does not provide any information about the location of the agent, and the location needs to be inferred using the history.
More details about this POMDP domain can be found in Appendix \ref{apx:ablationPOMDP}, where we also provide an ablation study across different values of $p = \{0.0, 0.3, 0.6, 0.9\}$. Here, we present the results for $p=0.3$.

 A sufficient statistic for any history is the ground truth state at the end of that history.  It can be observed  in Figures \ref{fig:POMDP_rep} and \ref{fig:POMDP_bonus} (in Appendix \ref{apx:ablationPOMDP}) that for the histories that are ending at a similar ground truth state (unknown to the agent), the proposed De-Rex method is able to provide similar representations for those histories. Additionally, the proposed method is also able to provide pseudo-counts for the underlying true state even when the true state is unknown.

In Figure \ref{fig:POMDP_toy2}, we provide an alternative way of assessing the quality of the learned representation for the histories.
It can be observed that De-Rex provides representations that can be effective at decoding the true underlying states.

\section{Related Work}
\label{sec:related}

Recent work by \citet{tang2022understanding} illustrates, under idealized assumptions, how self-predictive representation \citep{guo2020bootstrap,guo2022byol,schwarzer2020data} are  related to spectral decomposition.
Alternatively, \citet{ren2022spectral} consider explicitly decomposing the transition kernel, however, their method makes use of an estimated model, and their loss function requires access to the state distribution, which is typically not available directly.
Further, \citet{touati2021learning, lan2022novel} illustrate how the decomposition of long-term occupancy matrix could be useful for representation learning but they leave the exploration question open.
Some of these methods use a contrastive learning loss, which might be complimentary to the proposed non-contrastive loss \citep{garrido2022duality}.
Spectral properties of the transition matrix have also been used to learn options that aid exploration \citep{machado2017laplacian,machado2017eigenoption,liu2017eigenoption}. Further,  norms of a related quantity, known as successor representations, have been shown to be useful pseudo-counts \citep{machado2020count}.

We discuss other related works in detail in Appendix \ref{sec:apx:related}. Perhaps the work most relevant to ours is by \citet{wu2018laplacian} which discusses how to decompose the Laplacian of ${\bf P}_\pi$ to obtain the associated eigenfunctions.
Similar to ours, their method can also make use of deep networks and only requires access to transition samples.
However, their method is based on using a \textit{symmetric} Laplacian matrix, which may not be ideal depending on the underlying structure in the domain (for e.g., when $\bf P_\pi$ is not symmetric) being captured by the representations. 
Further, their method does not directly lends itself to estimating pseudo-counts for exploration either. 
In Section \ref{sec:results} we provide an empirical comparison against this method.

\section{Empirical Results}

\label{sec:results}

In this section, we aim to empirically investigate the properties of the proposed De-Rex algorithm for large-scale domains. 
Specifically, we aim to answer the two main questions pertaining to representation learning and exploration.

\textbf{\textit{Q1. Can De-Rex learn useful representations from rich observations?}}

To investigate this question, we use the DMLab 30~\citep{beattie2016deepmind} environment \citep{beattie2016deepmind}. DMLab is particularly useful here as it is a procedurally-generated, partially observable $3$-D world where properties such as object shapes, colors, and positions can vary every episode. The environments consist of $30$ diverse navigation and puzzle-solving tasks, where the state observations range from language instruction to pixel images from a first-person view, and rewards. The action set consists of a mix of discrete and continuous actions corresponding to motions for moving the agent's location, adjusting the visual angles, and controls for using various tools available in the environment. 

To make it more challenging, we consider the multi-task setting where there is a single agent that interacts with all the tasks at once and must learn a common representation across all of these tasks. In our setup, there is \textit{no explicit task identification} that is given to the agent, and so the agent must be able to deduce the task from its environment. We consider the following algorithms for comparison. Detailed discussions regarding implementation and hyperparameters for these are available in Appendix \ref{apx:implement}.

\textbf{RL Baseline:} As our base RL algorithm, we use V-MPO~\citep{song2019v}, a policy optimization method that had achieved state-of-the-art performance across several RL benchmarks involving both discrete and continuous tasks.
To avoid confounding factors, this baseline agent does not have any explicit representation learning loss.  

\textbf{Laplacian decomposition: }  We compare against the work by  \citet{wu2018laplacian} that uses spectral graph drawing to learn state representations by decomposing the Laplacian associated with the transition kernel.
As a baseline decomposition-based method, we use their proposed loss as an auxiliary task with the VMPO agent.

\textbf{De-Rex Representation: } Similar to the baseline decomposition-based method, we use the proposed auxiliary representation learning loss based on \eqref{eqn:loss} and \eqref{eqn:fullPSR} alongside the VMPO agent.
For these experiments, we focus only on understanding the usefulness of the representations being learned, without any additional exploration bonuses.

An empirical comparison of the algorithms for each task is presented in Figure \ref{fig:pergame}. 
In Appendix \ref{sec:moreres} we also provide aggregate and separate learning curves for each task.
It can be observed that the baseline VMPO that does not explicitly employ any representation learning losses, struggles to perform in the large-scale, rich observation setting. 
In comparison, the baseline method that uses decomposition of the Laplacian to learn state representations does slightly improve the performance.
However, the proposed De-Rex representation learning method that is designed to explicitly handle partial observability provides the most improvement.
We found that adding the De-Rex bonus for exploration did \textit{not} improve results for these tasks, as in the dense reward settings bonuses can distract more than be helpful.

\begin{figure}[!h]
    \centering
    \includegraphics[width=0.48\textwidth]{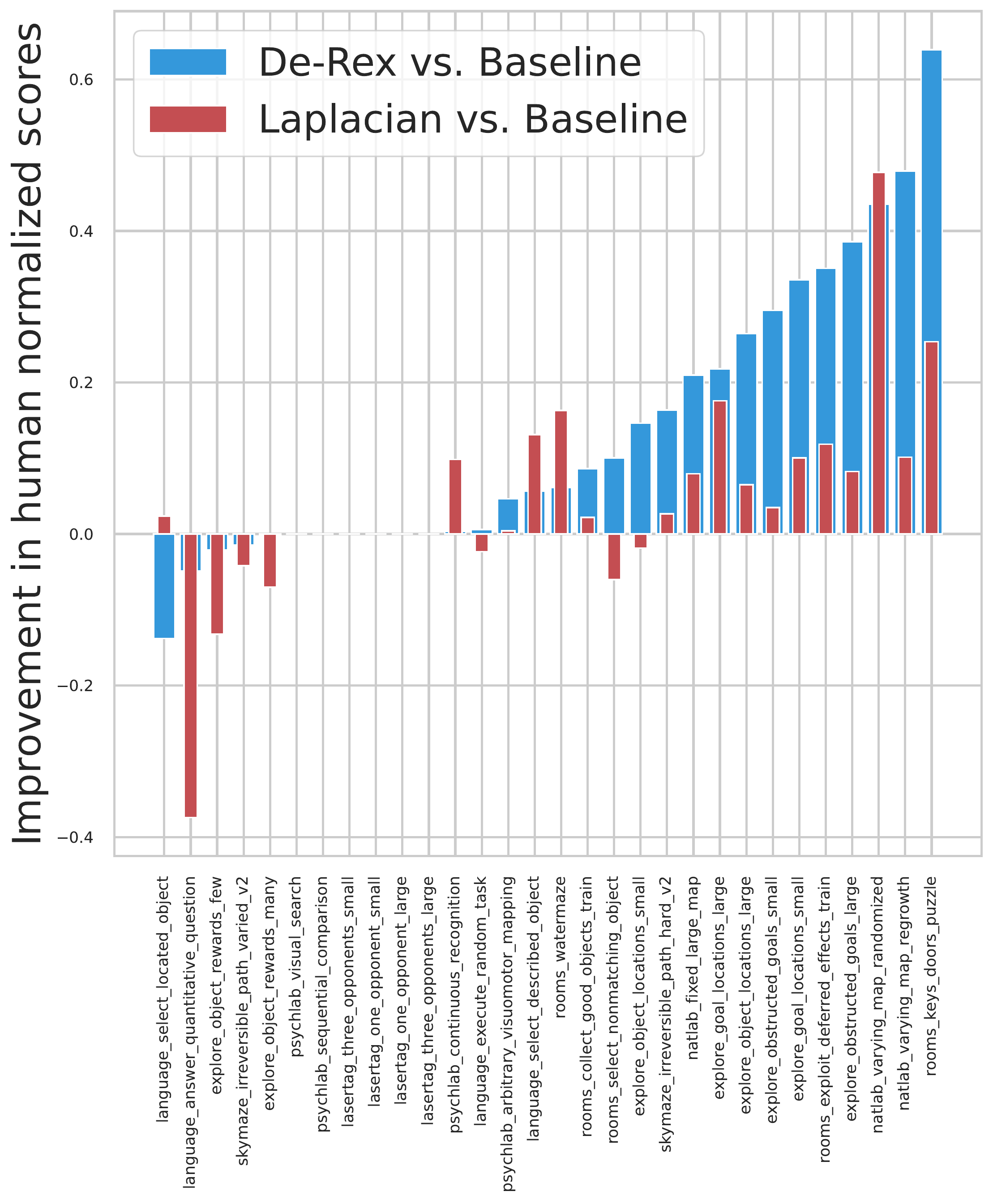}
    \caption{Per-game comparison of (a) De-Rex representation vs. VMPO baseline and (b) Laplacian representation vs. VMPO baseline. We compare the final human normalized scores (i.e., average of the last  $2\%$ of the total $2\cdot 10^{10}$ training steps). Each score is averaged over three seeds. 
    }
    \label{fig:pergame}
\end{figure}
\textbf{\textit{Q2. Can De-Rex explore complex environments with sparse rewards?}}
\begin{figure*}
    \centering
    \includegraphics[width=0.97\textwidth]{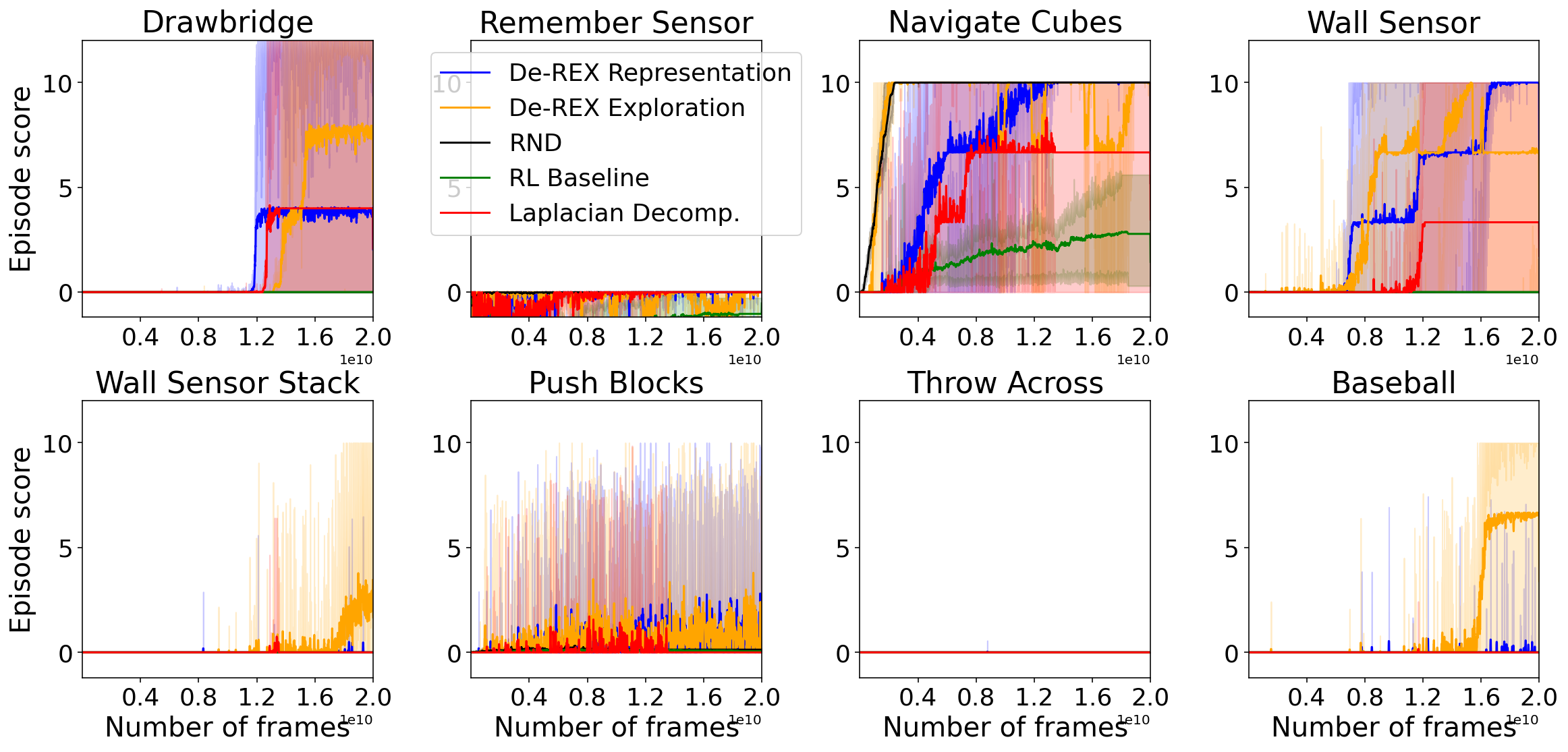}
    \caption{Learning curves for each task on the DM-Hard-8 suite \citep{gulcehre2019making} in the multi-task setup. The shaded area corresponds to the minimum and maximum across three seeds. High variance in these environments is due to  \textit{extremely sparse rewards} (e.g., only at the end of completing a sequence of complex tasks the agent is able to get any reward). Because of this, if the agent is able to explore well enough it solves the task, or else it fails completely. }
    \label{fig:dmhard_base}
\end{figure*}

Solving large-scale domains require both: learning representations and efficient exploration. 
To investigate this ability, we use the DMLab Hard-8 Suite~\citep{gulcehre2019making} that consists of $8$ hard exploration tasks in the DMLab environment~\citep{beattie2016deepmind} with sparse rewards, which require the agent to solve a sequence of non-trivial sub-tasks in the correct order to obtain the final reward.
Similar to DMLab-30, we make the setting more challenging by considering the multi-task setup.
For all the agents, we modify the reward at timestep $t$ by adding an algorithm-dependent bonus to encourage exploration, $ R_t' \coloneqq R_t + \lambda_b  \,\,\texttt{Bonus}_t,$
where $\lambda_b$ is a hyper-parameter. Besides comparing against the methods introduced before (that do not employ any exploration strategy and their $\texttt{Bonus}_t=0$ everywhere), we also compare the following methods.

\textbf{Random Network Distillation: } 
RND \citep{burda2018exploration} has emerged as a strong baseline exploration method that consists of training an observation encoder to predict a separate, fixed, randomly initialized encoder of the same observation. The prediction loss is then used to compute $\texttt{Bonus}_t$ for exploration. 

\textbf{De-Rex Exploration: } An important aspect of the proposed work is that the estimates for the pseudo-count of state visitations can be obtained easily. Besides using representation loss, we set $\texttt{Bonus}_t$ using the norm of the state/history representation at timestep $t$.

An empirical comparison of these methods is presented in Figure \ref{fig:dmhard_base}.
It can be observed that the base VMPO method fails to solve any task as it lacks any explicit way to encourage exploration.
Similarly, RND and the Laplacian decomposition \citep{burda2018exploration} methods only manage to solve \textit{one} DMHard-8 task. This can be attributed to the fact that the exploration bonus in RND does not include any representation learning component which is essential in the rich-observation setting. Whereas, the Laplacian decomposition method includes representation learning loss but has no explicit procedure to encourage exploration.  
De-Rex provides a single, unified, procedure for both representation learning and estimating bonuses to encourage exploration.
When only representation loss is used, \textit{De-Rex Representation} is still able to solve nearly \textit{two} tasks. When both the representation loss and exploration bonuses are used, \textit{De-Rex Exploration} is able to solve nearly \textit{four} tasks.
Appendix \ref{sec:moreres} contains additional ablation studies on the  DM-Hard-8 tasks using different hyperparameters combinations.

\section{Discussion and Conclusion}

We proposed an SVD based representation and exploration procedure. Our approach has three key
advantages. (a) Simplicity: The proposed representation learning loss is completely model-free and the exploration bonuses can be constructed readily by computing the norms of the representations. (b) Theory and empirical results: the proposed approach is grounded in the fundamental framework of SVD, and shows how SVD can be used not only for representation learning but also for exploration in challenging environments. (c) Extensibility: Our method can not only be used with other matrices of interest but also the proposed procedure can be used as an auxiliary loss with any future baseline (deep) RL algorithm.

\textbf{Limitations:} 
While we took the initial steps to lay out the analysis in the tabular setting, we believe an important future direction is to further our understanding by theoretically analyzing the function approximation setting as well. Further, investigating how to \textit{explicitly} account for rewards during representation learning could be fruitful (using the proposed method as an auxiliary loss induces \textit{implicit} dependency on rewards as the losses from policy/value functions also shape the representation).  

\section{Acknowledgement}

The work was done when YC was an intern at DeepMind. We thank Mohammad Gheshlaghi Azar, Hado van Hasselt, Tom Schaul, Mark Rowland, Bernardo Avila Pires, Bilal Piot, Daniele Calandriello, Michal Valko, David Abel, Marlos Machado, Doina Precup, and Michael Bowling, for several insightful discussions.

\clearpage

\bibliography{mybib}
\bibliographystyle{abbrvnat}

\clearpage
\include{appendix}

\end{document}

%% file: appendix.tex
\appendix

\setcounter{thm}{0}
\onecolumn

\section*{\centering Representations and Exploration for Deep Reinforcement Learning \\ using Singular Value Decomposition \\(Supplementary Material)}

\section{Extended Related Work}

\label{sec:apx:related}

\paragraph{Auxiliary Losses and Exploration Bonuses:}
Several prior works have analyzed properties of states features for TD learning \citep{parr2007analyzing,parr2008analysis,song2016linear}, studied different notions of optimality  \citep{nachum2018near,bellemare2019geometric,allen2021learning}, and the effect of auxiliary tasks on representation learning \citep{dabney2021value, lyle2021effect,lyle2022understanding,lyle2022learning}.
Specifically, similar to D-REX, auxiliary losses have been shown to help obtain higher expected returns quicker by learning better representations \citep{jaderberg2016reinforcement,mirowski2016learning,hessel2019multi,gregor2019shaping,guo2020bootstrap,erraqabi2022temporal}.
However, unlike these methods, D-REX neither requires negative sampling, nor any target networks for representation learning nor does it require reconstruction at the raw observation or representation level.

Auxiliary losses can also provide intrinsic motivation signals to drive exploration \citep{houthooft2016vime,pathak2017curiosity,burda2018exploration,azar2019world,badia2020never,guo2022byol}.
Such intrinsic bonuses are also related to pseudo-counts of state visitations \citep{bellemare2016unifying,ostrovski2017count,tang2017exploration}.
However, unlike these methods, D-REX does not require any sophisticated density estimation procedure and can directly use the norms of the learned representations to estimate pseudo-counts.
We refer the readers to  the work by \citet{weng2020exploration} for an
exhaustive review on various exploration strategies.

\paragraph{Decomposition Based Methods} 
Following the seminal work by \citet{mahadevan2005proto,mahadevan2007proto} several researchers have explored the utility of spectral representation and decompositions in decision making in MDPs \citep{petrik2007analysis,wang2021towards,behzadian2019fast},   POMDPs \citep{hsu2012spectral,roy2005finding,azizzadenesheli2016reinforcement,boots2011closing,grinberg2018learning}.
PAC learnability and regret guarantees have also been developed for MDPs \citep{misra2020kinematic,agarwal2020flambe} and POMDPs \citep{wang2022embed,kwon2021rl,guo2016pac,liu2022partially,zhan2022pac} using various rank conditions on the dynamics.
Complimentary to these results, our method is aimed at solving large-scale environments that present challenges both for representation learning and exploration.

\section{Proofs}

\begin{thm}
    If the eigenvalues of ${\bf P_\pi} \in \mathbb R^{|\mathcal S|\times |\mathcal S|}$
    are real and distinct,
    \footnote{we consider eigenvalues to be real and distinct to simplify the choice of `top-k' eigenvalues.} then for any $k$, the top-k eigenvectors of  ${\bf P_\pi}$ and $({\bf I} - \gamma {\bf P_\pi})^{-1}$ are the same. 
\end{thm}

\begin{proof}
\label{apx:ebfproof}
We first show that any eigenvector of $\bf P_\pi$ is also an eigenvector of $({\bf I}-\gamma{\bf P_\pi})^{-1}$. Subsequently, we will show that the ordering of these eigenvectors, based on their associated eigenvalues, are also the same.

Let $v \in \mathbb R^{|S|\times 1}$ be an eigenvector of $\bf P_\pi$ and $\lambda$ be its associated eigenvalue, i.e.,
\begin{align}
   {\bf P_\pi} v = \lambda v.
\end{align}
Now it can be observed that the eigenvectors for $\bf P_\pi$ and $({\bf I} - \gamma {\bf P_\pi})^{-1}$ are the same,
\begin{align}
    ({\bf I-}\gamma{\bf{\bf P_\pi}})^{-1}v &= \left(\sum_{t=0}^\infty \gamma^t{\bf{\bf P_\pi}}^t\right) v
    = \sum_{t=0}^\infty \left(\gamma\lambda\right)^t v.
\end{align}
Further, as $0 \leq\gamma<1$ and $|\lambda|\leq 1$,
\begin{align}
    ({\bf I}-\gamma{\bf{\bf P_\pi}})^{-1}v &= \frac{1}{1 - \gamma \lambda}v.
\end{align}
%
Let $\lambda ' \coloneqq 1/(1 - \gamma \lambda)$ be the eigenvalue associated with $v$ for the matrix $({\bf I}-\gamma{\bf{\bf P_\pi}})^{-1}$. As $|\lambda| \leq 1$,
%
$\lambda$ and $\lambda'$ are monotonously related, i.e., larger values of $\lambda$ imply larger values of $\lambda'$ and the ordering of the top-k eigenvectors for both $\bf P_\pi$ and $({\bf I}-\gamma{\bf P_\pi})^{-1}$ are the same.

\end{proof}
\begin{thm}
    ${\bf A_n}$ is an unbiased and a consistent estimator of ${\bf D}{\bf P_\pi}$, i.e.,
    \begin{align}
       \forall i,j, \quad \mathbb E\left[{\bf A_n}_{[i,j]}\right] = {\bf D}{\bf P_\pi}_{[i,j]}, && {\bf A_n}_{[i,j]} \overset{\text{a.s.}}{\longrightarrow} {\bf D}{\bf P_\pi}_{[i,j]}.
    \end{align}
\end{thm}

\begin{proof}
\label{apx:unbconsproof}

\textbf{Unbiased:}
\begin{align}
    \mathbb E\left[ \bf A_n \right] &=     \mathbb E\left[\frac{1}{n} \bf X^\top Y \right]
   = \mathbb E\left[\frac{1}{n} \bf X^\top \mathbb E[Y | X] \right].
\end{align}
In the tabular setting, as $\bf  \mathbb E[Y | X]$ is equal to the distribution of the next state $Y$ given interaction using $\pi$ from the state $X$, therefore $\bf  \mathbb E[Y | X] = \bf P_\pi$,
\begin{align}
    \mathbb E\left[\bf A_n \right] &\overset{(a)}{=} \mathbb E\left[\frac{1}{n} \bf X^\top X P_\pi \right]
    = \mathbb E\left[\frac{1}{n}\bf  X^\top X \right] \bf P_\pi. \label{eqn:int}
\end{align}
Now observe that because of one-hot encoding, $\frac{1}{n} \bf X^\top X$ is a diagonal matrix, where $i^{\text{th}}$ diagonal entry consists of the normalized visitation counts for the state $i$.
Therefore,
\begin{align}
    \mathbb E\left[\frac{1}{n}\bf  X^\top X \right]_{[i,i]} &= \mathbb E\left[\frac{1}{n} \sum_{j=1}^n  x_j x_j^\top \right]_{[i,i]}
    \\
    &= \mathbb E\left[\frac{1}{n} \sum_{j=1}^n  \mathrm 1_{\{s_j=s_i \wedge s_j=s_i\}} \right]
    \\
    &= \mathbb E\left[\mathrm 1_{\{S=s_i\}} \right]
    \\
    &= d(s_i).
\end{align}
Therefore, $\mathbb E\left[\frac{1}{n}\bf  X^\top X \right] = \bf D$, and  using \eqref{eqn:int},
\begin{align}
    \mathbb E\left[\bf A_n \right]
    &= \bf DP_\pi.
\end{align}

\textbf{Consistency:}
\begin{align}
    {\bf DP_\pi}_{[i,j]} &= {\bf D}_{[i,i]} {\bf P_\pi}_{[i,j]}
    \\
    &= d(s_i)\Pr(S'=s_j|S=s_i;\pi)
    \\
    &= \Pr(S'=s_j,S=s_i;\pi)
\end{align}
\begin{align}
    {\bf A_n}_{[i,j]} &= \frac{1}{n} {\bf X^\top Y}_{[i,j]}
   = \frac{1}{n}\sum_{k=1}^n (x_k y_k^\top)_{[i,j]}
\end{align}
As $x_k$ and $y_k$ are one-hot encoded vectors for the states $S_k$ and $S_k'$,
\begin{align}
    {\bf A_n}_{[i,j]} &= \frac{1}{n}\sum_{k=1} \mathrm 1_{\{S_k=s_i \wedge S_k'=s_j\}} 
\end{align}
Let $X_k = \mathrm 1_{\{S_k=s_i \wedge S_k'=s_j\}}$, then we know from the strong law of large numbers,
\begin{align}
\frac{1}{n}\sum_{k=0}^\infty X_k \overset{a.s.}{\longrightarrow} \mathbb E[\mathrm 1_{\{S_k=s_i \wedge S_k'=s_j\}} ] = \Pr(S'=s_j,S=s_i;\pi)
\end{align}    
Therefore, 
\begin{align}
\forall i,j, \quad    {\bf A_n}_{[i,j]} &\overset{a.s}{\longrightarrow} \Pr(S'=s_j,S=s_i;\pi)
\\
&= {\bf DP_\pi}_{[i,j]}.
\end{align}
\end{proof}

\begin{thm}
    If ${\bf P_\pi} {\bf P_\pi} ^\top$ and $\bf D$ are invertible,
    then for $\alpha_{s} \coloneqq ({\bf P_\pi} {\bf P_\pi}^\top)^{-1}_{[s,s]}$,
    \begin{align}
        \big\lVert f_{\bf U}(x_s)\big\rVert _{\bf \Lambda^{-1}}^2 = \frac{\alpha_s}{d(s)^2}. 
    \end{align}
\end{thm}

\begin{proof}
\label{apx:expproof}
First notice that,
\begin{align}
    \bf DP_\pi &= \bf U\Sigma V^\top 
    \\
    \bf DP_\pi P_\pi ^\top D^\top &= \bf U\Sigma V^\top V \Sigma^\top U^\top
    \\
    &= \bf U\Sigma^2 U^\top, &  \because \bf V^\top V =\bf I
    \\
        &= \bf U\Lambda U^\top & \text{where,} \,\, \bf \Lambda \coloneqq \bf \Sigma^2.
\end{align}

Recall that for any matrix $M$ with full column-rank, it's pseudo-inverse is given by $\bf M^\dagger \coloneqq \bf (M^\top M)^{-1}M^\top$ such that $\bf M^\dagger M = \bf I$.
As $\bf (D P_\pi P_\pi ^\top D^\top)^\top =\bf D P_\pi P_\pi ^\top D^\top $, 
\begin{align}
\bf (DP_\pi P_\pi ^\top D^\top)^\dagger &= \bf (D P_\pi P_\pi ^\top D^\top D P_\pi P_\pi ^\top D^\top )^{-1} D P_\pi P_\pi^\top D^\top \\
&= \bf (U\Lambda U^\top U\Lambda U^\top)^{-1} U\Lambda U^\top\\
&\overset{(a)}{=} \bf (U \Lambda^2 U^\top)^{-1} U\Lambda U^\top\\
&\overset{(b)}{=} \bf U \Lambda^{-2} U^\top U \Lambda U^\top \\
&= \bf U \Lambda^{-1} U^\top, \label{eqn:part1}
\end{align}
where (a) follows from the fact that $\bf U^\top U = \bf I$ and (b) follows from the fact that $\bf U^{-1} = \bf U^\top$.

Alternatively, under the assumption that $\bf P_\pi P_\pi ^\top$ and $\bf D$ are invertible,
\begin{align}
    \bf (DP_\pi P_\pi^\top D^\top )^\dagger
    &= \bf \left((D P_\pi P_\pi^\top D^\top)^\top (D P_\pi P_\pi^\top D^\top) \right)^{-1} (D P_\pi P_\pi^\top D^\top)^\top %
    \\
    &= \bf \left(D P_\pi P_\pi^\top D^2  P_\pi P_\pi^\top D \right)^{-1} D P_\pi P_\pi^\top D
    \\
    &= \bf D^{-1}( P_\pi P_\pi^\top)^{-1} D^{-2} ( P_\pi P_\pi^\top)^{-1} D^{-1} D P_\pi P_\pi^\top D
    \\
    &= \bf D^{-1}( P_\pi P_\pi^\top)^{-1} D^{-1}.  \label{eqn:part2}
\end{align}

Combining \eqref{eqn:part1} and \eqref{eqn:part2},
\begin{align}
  \bf  U \Lambda^{-1} {\bf U}^\top &= \bf D^{-1}( P_\pi P_\pi^\top)^{-1} D^{-1}.  
\end{align}
In the tabular setting, $x_{s} \in \mathbb R^{|S|\times 1}$ correspond to the one-hot feature encoding for the state $s$, and the linear function $f_U(x_{s}) = {\bf U}^\top x_{s}$
\begin{align}
    \langle f_U(x_{s}), f_U(x_{s}) \rangle_{\bf \Lambda^{-1}} &= x_{s}^\top  {\bf U} \Lambda^{-1}  {\bf U}^\top x_{s}  
    \\
    &= x_{s}^\top   {\bf D}^{-1}( {\bf P_\pi P_\pi}^\top)^{-1}  {\bf D}^{-1}  x_{s}  
    \\ 
    &\overset{(a)}{=} \frac{\alpha_{s}}{d(s)^2},
\end{align}
where, $\alpha_{s} = ( {\bf P_\pi P_\pi}^\top)^{-1}_{[s,s]}$, and the simplification in step (a) follows because $x_s$ is one-hot and $\bf D$ is a diagonal matrix.
\end{proof}

\begin{thm}
Gradient of $\mathcal L_{\operatorname{diag}}({\bf \widehat \Sigma}(\theta))$ is an unbiased estimator of  $\mathcal L_{\operatorname{diag}}({\bf  \Sigma}(\theta))$, and  gradient of $\mathcal L_{\operatorname{off}}({\bf \widehat \Sigma}(\theta))$ is in general a biased estimate of $\mathcal L_{\operatorname{off}}({\bf \Sigma}(\theta))$ , i.e.,
    \begin{align}
    \mathbb E\left[\partial_\theta \mathcal L_{\operatorname{diag}}\left(\widehat {\bf \Sigma}(\theta) \right)\right] &= 
    \partial_\theta \mathcal L_{\operatorname{diag}}({\bf \Sigma}(\theta)), 
    \\
    \mathbb E\left[\partial_\theta \mathcal L_{\operatorname{off}}\left(\widehat {\bf \Sigma}(\theta) \right)\right] &{\,\, \color{red}\neq} \,\,
    \partial_\theta \mathcal L_{\operatorname{off}}({\bf \Sigma}(\theta)), 
\end{align}
where the expectation is over the randomness of the mini-batch.
\end{thm}

\label{apx:biasproof}

\begin{proof}

We re-write \eqref{eqn:fullC} as
\begin{align}
    {\bf \Sigma}(\theta) \coloneqq \mathbb E \left[ f_{\theta_1}(S) g_{\theta_2}(S')^\top\right],
\end{align}
where the expectation is over the distribution specified by the $n$ sampled transition tuples in $\mathcal D$, i.e., ${\bf \Sigma} (\theta)$ is the sample average of $f_{\theta_1}(s_i) g_{\theta_2}(s_i')^\top$ across the $n$ samples in $\mathcal D$.

From \eqref{eqn:svdl2} and \eqref{eqn:batchC},
\begin{align}
    \mathbb E\left[\partial_\theta \mathcal L_{\operatorname{diag}}\left(\widehat {\bf \Sigma}(\theta) \right)\right] &=  \mathbb E \left[
    \frac{1}{k} \sum_{i=1}^k  \partial_\theta \widehat {\bf \Sigma}(\theta)_{[i,i]} \right]
    \\
    &\overset{(a)}{=}
    \frac{1}{k} \sum_{i=1}^k  \partial_\theta \mathbb E \left[ \widehat {\bf \Sigma}(\theta)_{[i,i]} \right]
    \\
    &=
    \frac{1}{k} \sum_{i=1}^k  \partial_\theta {\bf \Sigma}(\theta)_{[i,i]}
    \\
    &= \partial_\theta \mathcal L_{\operatorname{diag}}({\bf \Sigma}(\theta)), 
\end{align}
where (a) follows because we are looking at the case where the sampling distribution is fixed.

In the following we use a counter-example to show that the gradient for the mini-batch version of off-diagonal based loss does not provide an unbiased estimator in general.
For this counter-example, let the mini-batch size $b=1$. Therefore,
\begin{align}
    \mathbb E\left[\partial_\theta \mathcal L_{\operatorname{off}}\left(\widehat {\bf \Sigma}(\theta) \right)\right] &= \mathbb E\left[  \frac{1}{k^2-k} \sum_{\substack{i,j=1\\ i\neq j}}^k \partial_\theta \widehat{\bf \Sigma}(\theta)_{[i,j]}^2 \right],
    \\
    &=  \frac{1}{k^2-k} \sum_{\substack{i,j=1\\ i\neq j}}^k \partial_\theta \mathbb E\left[ \widehat{\bf \Sigma}(\theta)_{[i,j]}^2 \right],
    \\
    &\overset{(b)}{=} \frac{1}{k^2-k} \sum_{\substack{i,j=1\\ i\neq j}}^k \partial_\theta \mathbb E\left[ f_{\theta_1}(S){ g_{\theta_2}(S')^\top_{[i,j]}}^2\right],
    \\
    &{\color{red}\neq } \frac{1}{k^2-k} \sum_{\substack{i,j=1\\ i\neq j}}^k \partial_\theta \mathbb E\left[ f_{\theta_1}(S) g_{\theta_2}(S')^\top_{[i,j]}\right]^{\color{red}2},
    \\
    &= \partial_\theta \mathcal L_{\operatorname{off}}\left({\bf \Sigma}(\theta) \right),
\end{align}
where the step (b) follows because the mini-batch size being considered is $1$, and the step marked in {\color{red} red} \textit{fails} because for a random variable $X$, in general $\mathbb E[X^2] \neq \mathbb E[X]^2$.

\end{proof}

\begin{thm}
Gradient $\partial_\theta \widehat{ \mathcal L}_{\operatorname{off}}(\theta)$ is an unbiased estimator of $\partial_\theta \mathcal L_{\operatorname{off}}({\bf \Sigma}(\theta))$, i.e.,
    $$
    \mathbb E\left[\partial_\theta \widehat{ \mathcal L}_{\operatorname{off}}(\theta) \right]  = 
    \partial_\theta \mathcal L_{\operatorname{off}}({\bf \Sigma}(\theta)). $$
\end{thm}

\begin{proof}

\begin{align}
    \widehat {\mathcal L}_{\operatorname{off}}(\theta) &\coloneqq \underbrace{\frac{1}{(k^2-k)} \sum_{\substack{i,j=1\\ i\neq j}}^k \widehat {\bf \Sigma}(\theta)_{[i,j]} \,\, \operatorname{sg}\left(\widehat{\bf \widehat \Sigma}(\theta)_{[i,j]}\right)}_{\color{blue} (I)} + \underbrace{\frac{1}{(k^2-k)} \sum_{\substack{i,j=1\\ i\neq j}}^k \widehat{\bf \widehat \Sigma}(\theta)_{[i,j]} \,\, \operatorname{sg}\left(\widehat{\bf \Sigma}(\theta)_{[i,j]}\right)}_{\color{red} (II)},
\end{align}

\begin{align}
    \mathbb E\left[\partial_\theta{\color{blue} (I)} \right] &\coloneqq \frac{1}{(k^2-k)} \sum_{\substack{i,j=1\\ i\neq j}}^k \mathbb E \left[ \partial_\theta \widehat {\bf \Sigma}(\theta)_{[i,j]} \,\, \operatorname{sg}\left(\widehat{\bf \widehat \Sigma}(\theta)_{[i,j]}\right) \right],
    \\
    &\overset{(a)}{=} \frac{1}{(k^2-k)} \sum_{\substack{i,j=1\\ i\neq j}}^k \mathbb E \left[ \partial_\theta \widehat {\bf \Sigma}(\theta)_{[i,j]} \right] \mathbb E\left[ \widehat{\bf \widehat \Sigma}(\theta)_{[i,j]} \right],
    \\
    &= \frac{1}{(k^2-k)} \sum_{\substack{i,j=1\\ i\neq j}}^k \mathbb \partial_\theta \mathbb E \left[  \widehat {\bf \Sigma}(\theta)_{[i,j]} \right] \mathbb E\left[ \widehat{\bf \widehat \Sigma}(\theta)_{[i,j]} \right],
    \\
    &= \frac{1}{(k^2-k)} \sum_{\substack{i,j=1\\ i\neq j}}^k \mathbb \partial_\theta  {\bf \Sigma}(\theta)_{[i,j]}  {\bf \Sigma}(\theta)_{[i,j]},
    \\
    &= \frac{1}{2} \partial_\theta \left( \frac{1}{(k^2-k)} \sum_{\substack{i,j=1\\ i\neq j}}^k  {\bf \Sigma}(\theta)_{[i,j]}^2\right),
    \\
    &= \frac{1}{2} \partial_\theta \mathcal L_{\operatorname{off}}\left({\bf \Sigma}(\theta) \right),
\end{align}
where step (a) follows because $\widehat {\bf \Sigma}(\theta)$ and $\widehat{\widehat {\bf \Sigma}}(\theta)$ are independently computed using different batches of data.
Similarly, it can be observed that $ \mathbb E\left[\partial_\theta{\color{blue} (I)} \right] =  \mathbb E\left[\partial_\theta{\color{red} (II)} \right]$.
Therefore,
\begin{align}
     \mathbb E\left[\partial_\theta \widehat{ \mathcal L}_{\operatorname{off}}(\theta) \right]  &= \mathbb E\left[\partial_\theta{\color{blue} (I)} \right] +  \mathbb E\left[\partial_\theta{\color{red} (II)} \right]
     \\
     &=     \partial_\theta \mathcal L_{\operatorname{off}}({\bf \Sigma}(\theta)).
\end{align}
\end{proof}

\section{Experimental details}

\label{apx:implement}

We provide further details on the algorithmic and implementation details of deep RL experiments in the paper.

\paragraph{DM-Lab-30 environments.} DM-Lab-30 is a set of visually complex maze navigation tasks \citep{beattie2016deepmind}. The only input
the agent receives is an image of the first-person view of the
environment, a natural language instruction (in some tasks),
and the reward.  Agents can provide multiple simultaneous actions to control movement (forward/back,
strafe left/right, crouch, jump), looking (up/down, left/right) and tagging (in laser
tag levels with opponent bots). DM-Lab-30 has proved challenging for model-free RL algorithms \citep{beattie2016deepmind,guo2020bootstrap}.

\paragraph{DM-Hard-8 environments: } DM-Hard-8 is a set of similar visually complex maze navigation tasks with sparse reward, and is commonly used as a test bed for exploration \citep{guo2022byol}.
This benchmark comprises of 8 hard exploration tasks aimed at  emphasizing the challenges encountered when learning from sparse rewards in 
a procedurally-generated 3-D world with partial observability, continuous control, and highly variable
initial conditions.
Each task requires reaching an apple in the environment through interaction with specific objects in its environment.
Only on reaching the apple provides a reward to the agent.
Being procedurally generated, properties
such as object shapes, colors, and positions are different in every episode. The agent sees only the
first-person view from its position.

\paragraph{Agent architecture.} Since the environment is partially observable, we adopt a recurrent network architecture following prior work such as \citep{beattie2016deepmind,song2019v,guo2020bootstrap}. In particular, the agent uses a convolutional neural network inside $\mathcal F_1 : \mathcal O \rightarrow \mathbb R^k$ to process the image in the input $o_t \in \mathcal O$ into latent embedding.
The observation contains (i) the force the agent’s hand is currently applying
as a fraction of total grip strength, a single scalar between 0 and 1; (ii) a boolean indicating
whether the hand is currently holding an object; (iii) the distance of the agent's hand from the
main body; (iv) the last action taken by the agent; each of these are embedded using a linear
projection to 20 dimensions followed by a ReLU. (v) the previous reward obtained by the
agent, passed by a signed hyperbolic transformation; (vi) a text instruction specific to the
task currently being solved, with each word embedded into 20 dimensions and processed
sequentially by an LSTM \citep{hochreiter1997long} to an embedding of size 64. All of these quantities, along with 
the output of $\mathcal F_1$, are concatenated and passed through a linear layer to an embedding of size 512.
An LSTM  $\mathcal F_2$ with the embedding size 512 processes each of these observation representations
sequentially to form the recurrent history representation.
The core output of the LSTM is the belief $b_t=\mathcal F_2(h_t)\in\mathbb{R}^k$. Using the notation from the main paper, we can understand the belief as a function of the history $h_t=(o_0,a_0...o_t)$, thanks to the recurrent nature of the agent architecture.

Finally, we calculate value functions $v_\theta(b_t)$ and policy $\pi_\theta(a_t|b_t)$ based on the belief using MLP heads.
Further, all the algorithms are built upon VMPO with a similar reward normalization scheme as in RND \citep{burda2018exploration} and we normalize the raw rewards by an EMA estimate of its standard deviation. See Appendix A.3 in the work by \citet{guo2022byol} for the exact procedure.
For DM-Hard8 environments, rewards and bonuses are normalized separately before combining them using $\lambda_b$. 
That is, let $R_t$ and $\text{Bonus}_t$ be the normalized reward and bonus, then the effective reward for timestep $t$ is,
$R_t + \lambda_b  \,\,\text{Bonus}_t,$
where $\lambda_b$ is a hyper-parameter.

When an algorithm makes use of additional representation loss $\text{Rep}_{\text{loss}}$ then it is combined with the $\text{RL}_{\text{loss}}$ as following,
$\text{RL}_{\text{loss}} + w_\text{loss} \text{Rep}_{\text{loss}},$
where $w_\text{loss}$ is a hyper-parameter, and $\text{RL}_{\text{loss}}$ is the loss provided by the baseline VMPO for learning the policy and critic.

\paragraph{Baseline VMPO.} VMPO is a model-free RL algorithm \citep{song2019v} that achieves state-of-the-art performance in a number of challenging RL benchmarks. At its core, VMPO is a policy optimization algorithm. Importantly, the construction of VMPO updates are based on a number of prior inventions in the policy optimization literature, such as trust region optimization \citep{schulman2015trust,schulman2017proximal}, regression-based updates and adjusting trust regions based on primal-dual optimizations \citep{abdolmaleki2018maximum} and multi-step estimation of advantage functions \citep{espeholt2018impala}. We refer readers to \citep{song2019v} for a complete description of the algorithm and its hyper-parameters. The original VMPO algorithm adopts pixel-control as an auxiliary objective \citep{beattie2016deepmind} in the DM-Lab-30 experiments; we do not apply such an auxiliary loss in our experiment to avoid confounding effects. Further, there is no exploration bonus. That is, $\lambda_b = w_\text{loss} = 0$.

\paragraph{De-Rex.} 
\begin{figure}[t]
    \centering
    \includegraphics[width=0.7\textwidth]{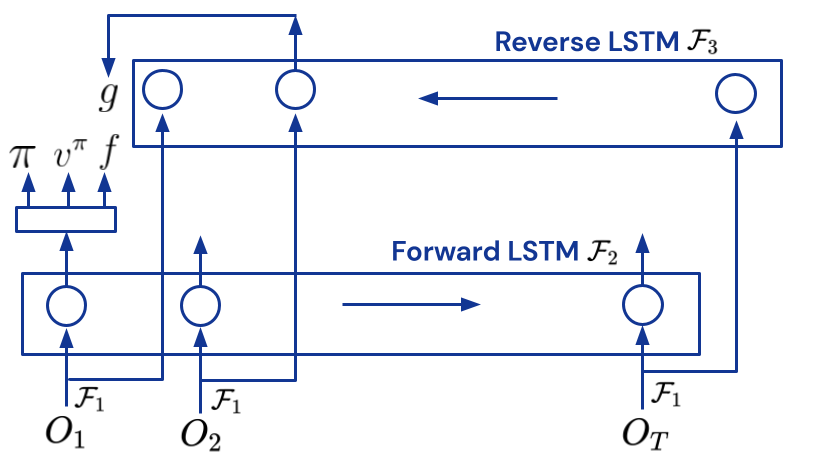}
    \caption{Recall that observation $O_t$ contains the previous action $A_{t-1}$. Let $\mathcal F_1(O_i)$ be the embedding for the observation at timestep $i$.
Let the belief state $b_t = \mathcal F_2\left(\mathcal F_1(O_0), \mathcal F_1(O_1), ... ,\mathcal F_1(O_t) \right)$ obtained using a forward LSTM.
Let $\mathcal F_3$ be a reverse LSTM, such that $ b_t' = \mathcal F_3(\mathcal F_1(O_T), \mathcal F_1(O_{T-1}), ..., \mathcal F_1(O_t))$. Then  $f_{\theta_1}(h_t) \in \mathbb R^k$ in \eqref{eqn:fullPSR} is defined to be $f_{\theta_1}(b_t)$, and $g_{\theta_2}(\tau_t)$ is defined to be $g_{\theta_2}(b_{t+1}')$.
}
    \label{fig:arch}
\end{figure}

This is the proposed method that provides a unified way to learn representation and also obtain pseudo counts for exploration.
\textit{De-Rex representation} only employs the representation learning loss as an auxiliary task, and there is no exploratory bonus, i.e., $\lambda_b=0$.
The representation loss $\text{Rep}_{\text{loss}} = \mathcal L({\bf \Sigma}(\theta))$ is based on \eqref{eqn:loss} and is combined with the VMPO baseline.
For implementing \eqref{eqn:fullPSR}, we base the functions $f_{\theta_1}$ and $g_{\theta_2}$ on forward and backward LSTMs, as illustrated in  Figure \ref{fig:arch}.
We observed that defining the representation loss as an auxiliary loss on a separate projection/transformation of belief $b_t$ is more beneficial than directly applying the loss on $b_t$, which is being used directly by the policy and the value functions.

In comparison, \textit{De-Rex Exploration} employs both representation learning loss and the exploratory bonus.
Here the $\text{Bonus}_t$ corresponds to the norm of the representation $f_{\theta_1}(h_t) \in \mathbb R^k$ as discussed in Section \ref{sec:PSR}.
In Tables \ref{tab:dmlabderex} and \ref{tab:dmhardderex} we provide more details about the hyper-parameters.

\begin{table}[h]
    \centering
    \begin{tabular}{c|c}
        Hyper-parameter & Value  \\
        \hline
        \\
        $\lambda_r$ & 10.0 (Tuned in smaller experiments)\\
        $\lambda_b$ & (N/A) \\
        $w_\text{loss}$ & 0.01 (Hypersweep $\{1, 0.1, 0.01, 0.001\}$)\\
        $k$ & 512 (Hypersweep $\{1024, 512, 64\}$) \\
        Batch Size & 256 
    \end{tabular}
    \caption{Hyper-parameters for De-Rex Representation for DMLab30}
    \label{tab:dmlabderex}
\end{table}

\begin{table}[h]
    \centering
    \begin{tabular}{c|c}
        Hyper-parameter & Value  \\
        \hline
        \\
        $\lambda_r$ & 10.0\\
        $\lambda_b$ & 0.001 (Hypersweep $\{1, 0.1, 0.01, 0.001\}$)\\
        $w_\text{loss}$ & 0.01 (Hypersweep $\{1, 0.1, 0.01, 0.001\}$)\\
         $k$ & 64 (Hypersweep $\{512, 256, 64, 8\}$) \\
        Batch Size & 256
    \end{tabular}
    \caption{Hyper-parameters for De-Rex Exploration for DMHard-8}
    \label{tab:dmhardderex}
\end{table}

\paragraph{Laplacian graph drawing \citep{wu2018laplacian}}: 
This is the baseline decomposition method that decomposes the Laplacian associated with the transition matrix.
Let $u$ and $v$ be state and next-state pairs, then their proposed loss is given by
\begin{align}
    \mathcal L_\text{diag} &\coloneqq \frac{1}{nk}\sum_{i=1}^n\sum_{j=1}^k (f_{\theta_1}(u_i)_j - f_{\theta_1}(v_i)_j)^2
    \\
    \mathcal L_\text{off} &\coloneqq \frac{1}{nk^2}\sum_{i=1}^n\sum_{j=1}^k\sum_{l=1}^k (f_{\theta_1}(u_i)_j f_{\theta_1}(u_i)_l - \delta_{j,l}) (f_{\theta_1}(v_i)_j f_{\theta_1}(v_i)_l - \delta_{j,l})
    \\
    \text{Rep}_\text{loss} &\coloneqq  \mathcal L_\text{diag} + \lambda_r \mathcal L_\text{off}
\end{align}
While the original formulation was defined for MDP, where $u$ and $v$ correspond to the state and the next state. However, as we  are dealing with POMDPs, we let the state be the history so far, therefore we set $u = h$ and $v= h'$.
The architecture for the function $f_{\theta_1}$ is the same as the one used for De-Rex.
This Laplacian baseline does not provide any means for exploration and hence $\lambda_b=0$ throughout.
For $\lambda_r$, we searched between $[1,10]$ and found $10$ to work the best. Other hyper-parameters were similar to De-Rex.

\paragraph{Random Network Distillation (RND).} RND \citep{burda2018exploration} provides an exploration strategy that generates an exploration bonus based on the prediction between the current network $\xi_\theta(o_t)\in\mathbb{R}^k$ and a randomly initialized network $\xi_{\theta_0}(o_t)\in\mathbb{R}^k$. Note that $\theta_0$ is randomly initialized and is not trained over time. The current network $\theta$ is trained to minimize the squared error $l_\theta(o_t)=\left\lVert \xi_\theta(o_t) - \xi_{\theta_0}(o_t)\right\rVert_2^2$. Meanwhile, $l_\theta(o_t)$ is also used as an exploration bonus.

\section{Ablation Study for 4-Rooms POMDP Domain }
\label{apx:ablationPOMDP}

In Figures \ref{fig:tabular4room} and \ref{fig:4roomfa} the representations learned and the bonuses obtained were visualized for the fully-observable setting.
In this section, we aim to further extend this qualitative assessment by performing ablations on the partially observable setting.
To restrict the focus on just the representation learning component and the bonuses constructed, we consider a partially observable extension of the 4 rooms domains.
This provides a controlled environment to test the effectiveness and limitations of the proposed De-Rex method.

\textbf{4 rooms (POMDP): } In this domain, observation $O_t$ at time $t$ corresponds to an  image of size $(30, 30, 3)$ that has the top-down view of the domain (similar to Figure \ref{fig:4roomfa} (B)), albeit that with probability $p$ the agent's marker (red dot) is completely hidden.
Therefore, when the marker is hidden, the observation $O_t$ does not provide any information about the location of the agent, and the location needs to be inferred using the history.
Naturally, a higher value of $p$ induces a higher degree of partial observability, and $p=0$ indicates full observability (the proposed De-Rex method does not assume full observability and still uses history-dependent representations).
The action set corresponds to the four directions of movement, and the horizon length is $250$.

Therefore, this domain provides a controlled setup to understand the method across different degrees of partial observability.
In the following, we provide qualitative analysis of the representations and bonuses learned by our method, across different values of $p$.
Data was collected using a uniform random policy.

\begin{figure}
    \centering
    \includegraphics[width=0.24\textwidth]{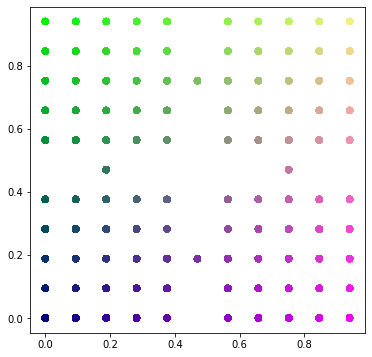}
    \includegraphics[width=0.24\textwidth]{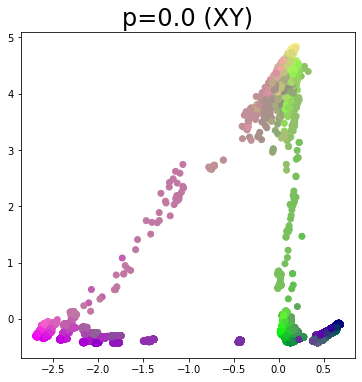}
    \includegraphics[width=0.24\textwidth]{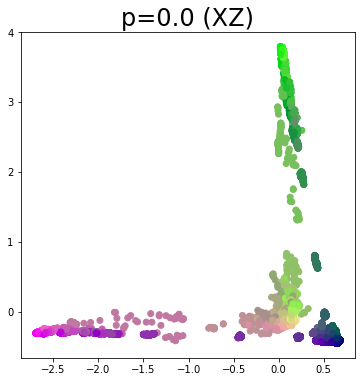}
    \includegraphics[width=0.24\textwidth]{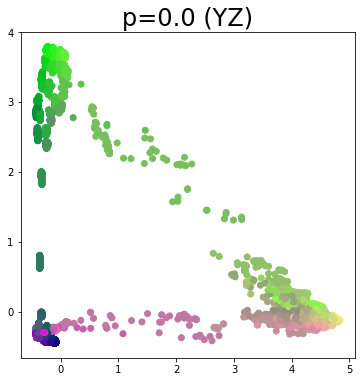}
    \\
    \hspace{0.24\textwidth}
    \includegraphics[width=0.24\textwidth]{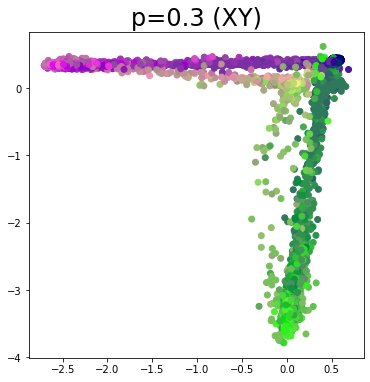}
    \includegraphics[width=0.24\textwidth]{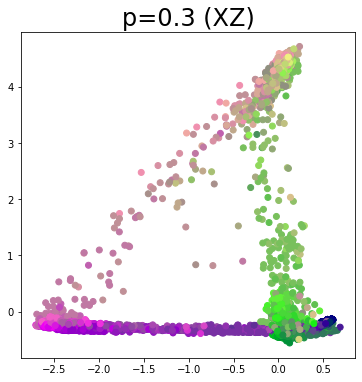}
    \includegraphics[width=0.24\textwidth]{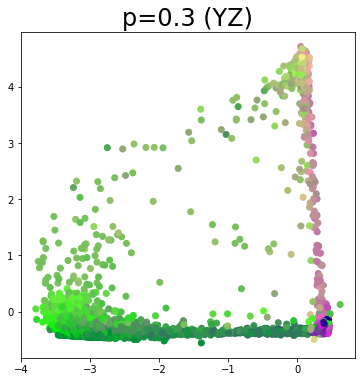}
    \caption{This plot is analogous to Figure \ref{fig:tabular4room}(B) and  Figure \ref{fig:tabular4room}(C) and extends the visualization technique to the POMDP setting. We plot the representations for all possible (partial) histories in the data collected \textbf{Top-Left} Ground truth color to represent the underlying true state $s$. \textbf{Top-row} Illustration of the learned three-dimensional representations for the histories in the 4-rooms (POMDP) setting where $p=0.0$. To view 3D representations, we plot slices (XY, XZ, YZ) of the axes. \textbf{Bottom-row}  Illustration of the learned three-dimensional representations for the histories in the 4-rooms (POMDP) setting where $p=0.3$. To view 3D representations, we plot slices (XY, XZ, YZ) of the axes. \textbf{Legend:} (Partial) histories that have $s$ as the true underlying state at the end of their sequence of observations have their representations  share the same color as the reference color for $s$. Therefore, as there are multiple histories that end at a given state, for a single color of the true state there are multiple learned representations with the same color. }
    \label{fig:POMDP_rep}
\end{figure}
\paragraph{Representations:}
In the fully observable case, the proposed method used a representation dimension of 2 (Figure \ref{fig:4roomfa}). For the partially observable setting, we found representation dimensions of $3$ and $4$ to be useful for $p<0.5$ and $p>0.5$ respectively. In Figure \ref{fig:POMDP_rep} we illustrate the effectiveness of learning representations using the De-Rex approach. Architecture for our $f$ and $g$ function follows the one illustrated in Figure \ref{fig:arch}, where the observation are individually processed via a conv-net, before passing it to a forward and a reverse LSTM to encode histories and futures, respectively. Outputs of these LSTMs are then subjected to a one-layer MLP to form $f$ and $g$.

A sufficient statistic for any history is the ground truth state at the end of the history. It can be observed that for the histories that are ending at similar ground truth state (unknown to the agent), the proposed De-Rex method is still able to provide similar representations for those histories. 
\paragraph{Bonuses:} In Figure \ref{fig:POMDP_bonus} we plot the bonuses obtained using the De-Rex procedure. It can be observed that when partial observability is not extreme, the proposed method is able to provide pseudo-counts for the underlying true state even when the true state is unknown. This can be attributed to the fact that De-Rex provides a joint mechanism for both learning a representation to capture the underlying sufficient statistic, and for estimating pseudo-counts using those represetnations.
\begin{figure}
    \centering
    \includegraphics[width=0.24\textwidth]{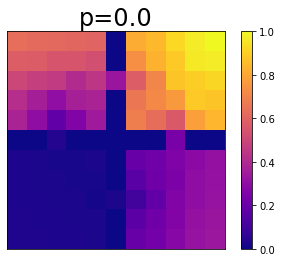}
    \includegraphics[width=0.24\textwidth]{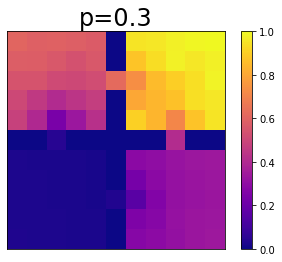}
    \\
    \includegraphics[width=0.24\textwidth]{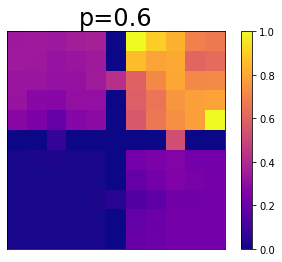}
    \includegraphics[width=0.24\textwidth]{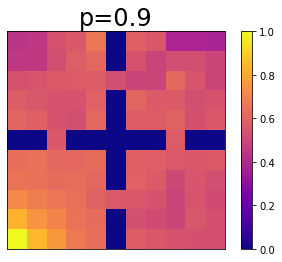}
    \caption{This plot is analogous to Figure \ref{fig:tabular4room}(D) and  Figure \ref{fig:4roomfa}(D) and extends the visualization technique to the POMDP setting. The ground truth state visitation is the same as in Figure \ref{fig:tabular4room}(A). \textbf{Top-left} Bonuses constructed using learned representations of the history when $p=0.0$. \textbf{Top-right} Bonuses constructed using learned representations of the history when $p=0.3$. \textbf{Bottom-left} Bonuses constructed using learned representations of the history when $p=0.6$. \textbf{Bottom-right} Bonuses constructed using learned representations of the history when $p=0.9$. \textbf{Legend:} Brighter color indicates higher value bonus. Colors for each of the true state $s$ is computed based on the average of bonuses from all (partial) histories that have $s$ as the true underlying state at the end of their sequence of observations.   }
    \label{fig:POMDP_bonus}
\end{figure}
\paragraph{Belief Distribution: } 
While Figure \ref{fig:POMDP_rep} provides one way to assess the quality of the learned distributions, it can be used to visualize representations in higher dimensions.
Therefore, in Figures \ref{fig:POMDP_belief0}, \ref{fig:POMDP_belief3}, \ref{fig:POMDP_belief6}, and \ref{fig:POMDP_belief9} we provide an alternate way to assess the quality of the learned representations.

For these figures, we learned a classification model (using a single-layer neural network) that aims to classify the true underlying state at the end of the observation sequence of a given (partial) history using the representation of that (partial) history. 
This provides us with a proxy for how well can the learned representations be used as to decode beliefs over the underlying true state.

As expected, when $p=0.0$ there is no partial observability and agent learns to use (ignore) history to obtain representations which can be used to precisely decode the belief over the ground truth location of the agent.
As the degree of partial observability increases, it can be observed that the belief distribution get more diffused as there is less certainty about the agent's location. Nonetheless, De-Rex provides representations that can be used to decode true underlying states within close proximity. 

\clearpage
\begin{figure}[h]
    \centering
    \includegraphics[width=0.1\textwidth]{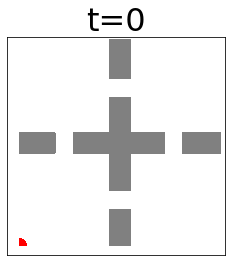}
    \includegraphics[width=0.1\textwidth]{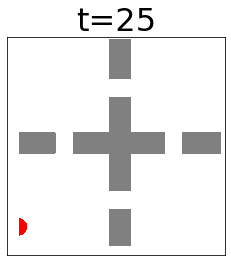}
    \includegraphics[width=0.1\textwidth]{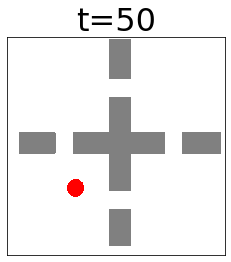}
    \includegraphics[width=0.1\textwidth]{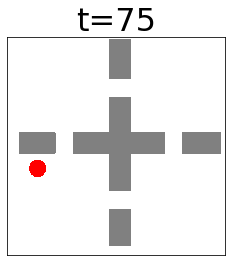}
    \includegraphics[width=0.1\textwidth]{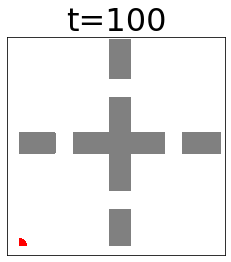}
    \includegraphics[width=0.1\textwidth]{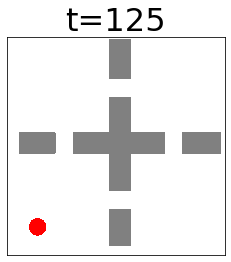}
    \includegraphics[width=0.1\textwidth]{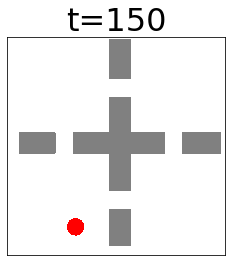}
    \includegraphics[width=0.1\textwidth]{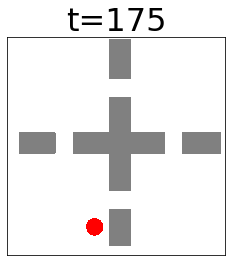}
    \includegraphics[width=0.1\textwidth]{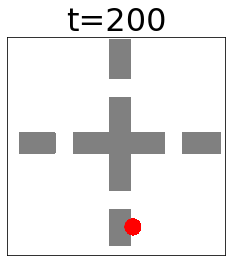}
    \\
    \includegraphics[width=0.1\textwidth]{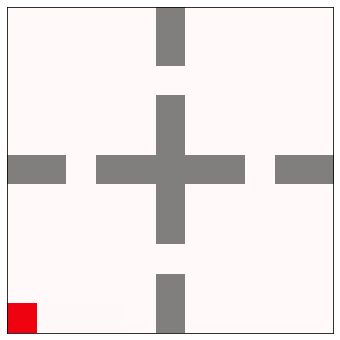}
    \includegraphics[width=0.1\textwidth]{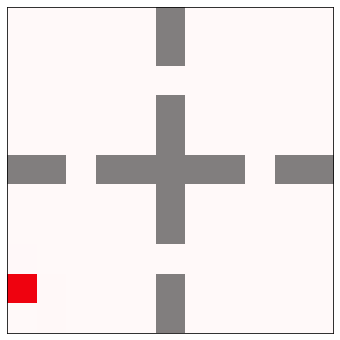}
    \includegraphics[width=0.1\textwidth]{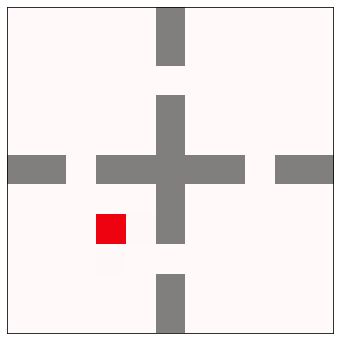}
    \includegraphics[width=0.1\textwidth]{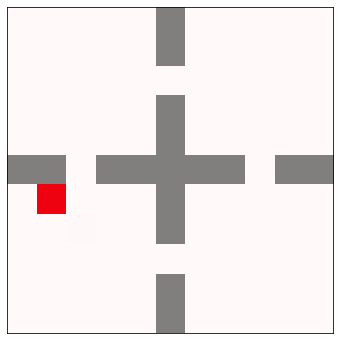}
    \includegraphics[width=0.1\textwidth]{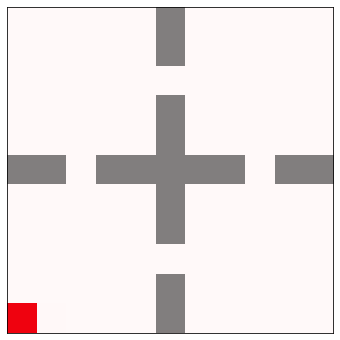}
    \includegraphics[width=0.1\textwidth]{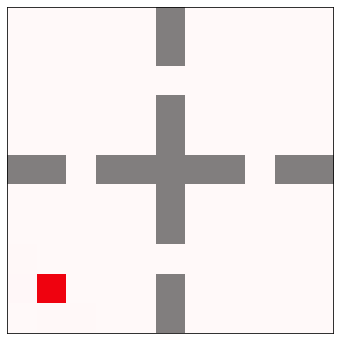}
    \includegraphics[width=0.1\textwidth]{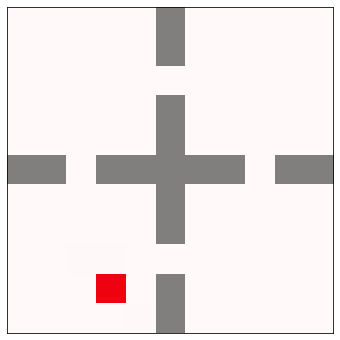}
    \includegraphics[width=0.1\textwidth]{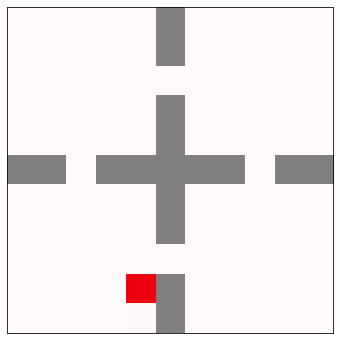}
    \includegraphics[width=0.1\textwidth]{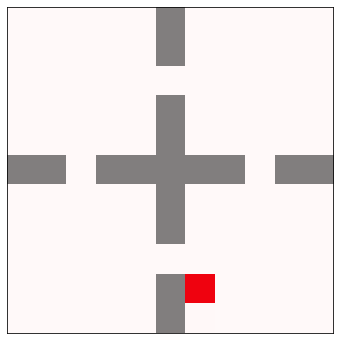}
    \caption{Beliefs for a given trial under the POMDP setting with $p=0.0$. }
    \label{fig:POMDP_belief0}
\end{figure}

\begin{figure}[h]
    \centering
    \includegraphics[width=0.1\textwidth]{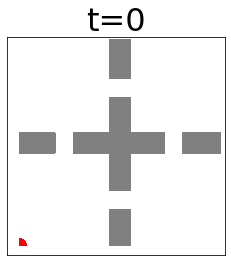}
    \includegraphics[width=0.1\textwidth]{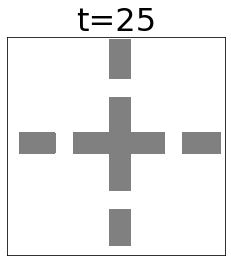}
    \includegraphics[width=0.1\textwidth]{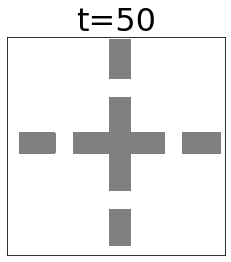}
    \includegraphics[width=0.1\textwidth]{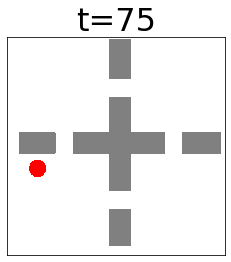}
    \includegraphics[width=0.1\textwidth]{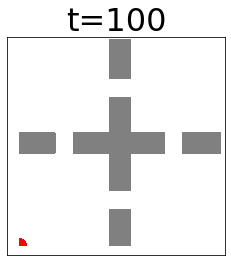}
    \includegraphics[width=0.1\textwidth]{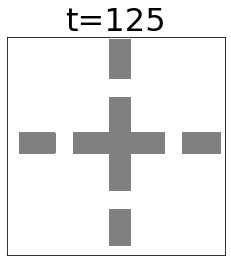}
    \includegraphics[width=0.1\textwidth]{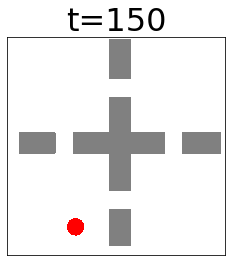}
    \includegraphics[width=0.1\textwidth]{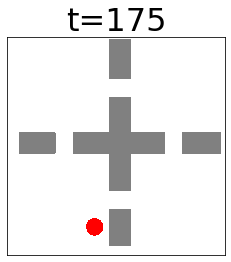}
    \includegraphics[width=0.1\textwidth]{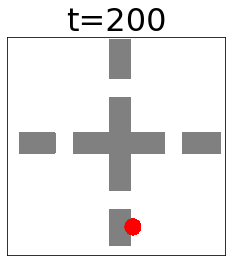}
    \\
    \includegraphics[width=0.1\textwidth]{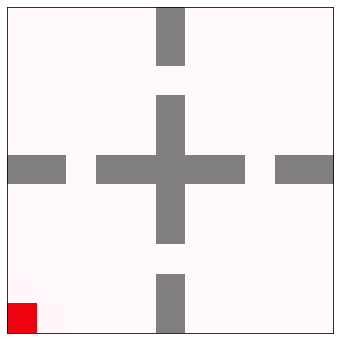}
    \includegraphics[width=0.1\textwidth]{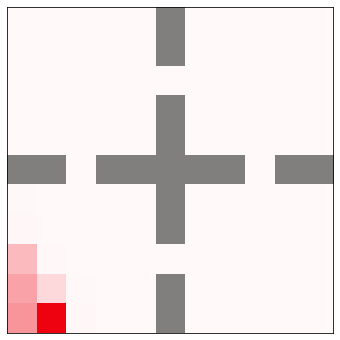}
    \includegraphics[width=0.1\textwidth]{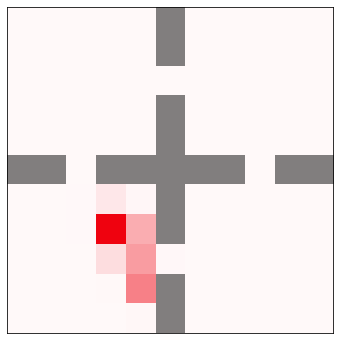}
    \includegraphics[width=0.1\textwidth]{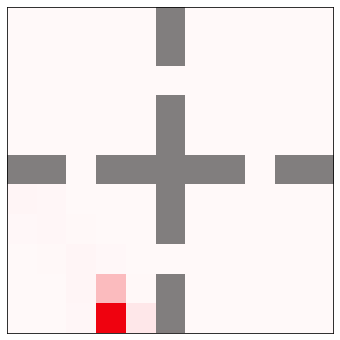}
    \includegraphics[width=0.1\textwidth]{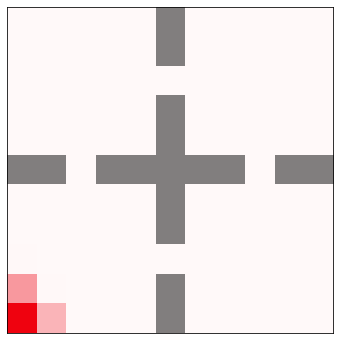}
    \includegraphics[width=0.1\textwidth]{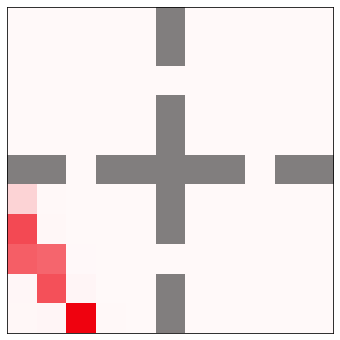}
    \includegraphics[width=0.1\textwidth]{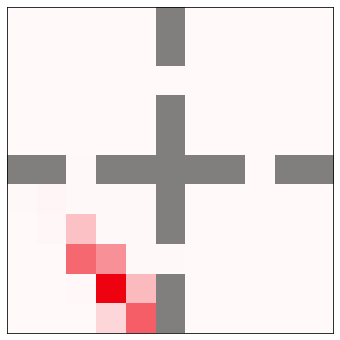}
    \includegraphics[width=0.1\textwidth]{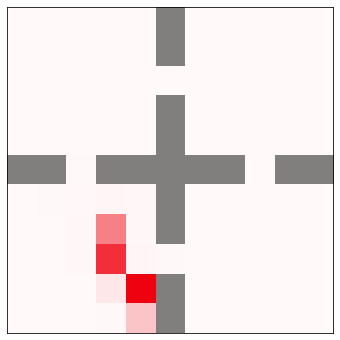}
    \includegraphics[width=0.1\textwidth]{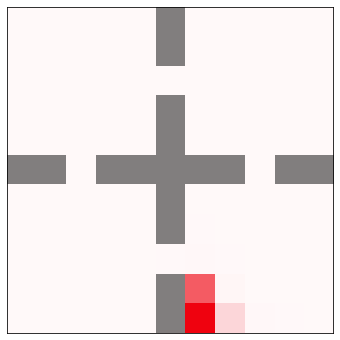}
    \caption{Beliefs for a given trial under the POMDP setting with $p=0.3$. }
    \label{fig:POMDP_belief3}
\end{figure}

\begin{figure}[h]
    \centering
    \includegraphics[width=0.1\textwidth]{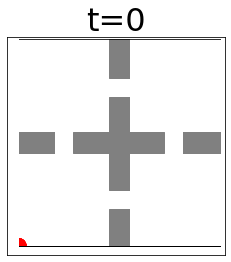}
    \includegraphics[width=0.1\textwidth]{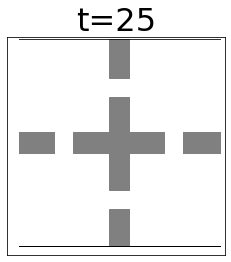}
    \includegraphics[width=0.1\textwidth]{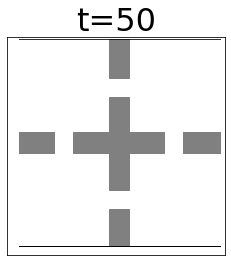}
    \includegraphics[width=0.1\textwidth]{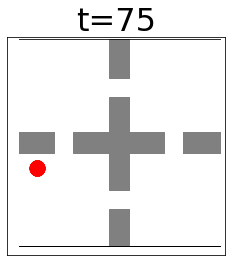}
    \includegraphics[width=0.1\textwidth]{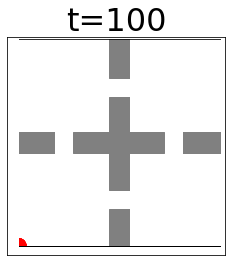}
    \includegraphics[width=0.1\textwidth]{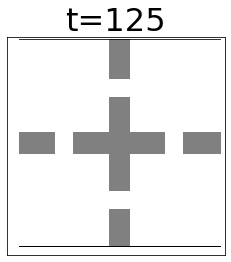}
    \includegraphics[width=0.1\textwidth]{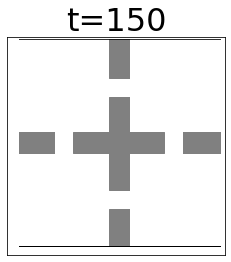}
    \includegraphics[width=0.1\textwidth]{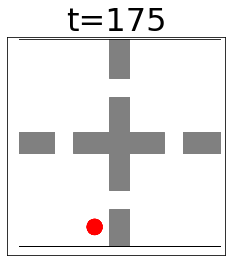}
    \includegraphics[width=0.1\textwidth]{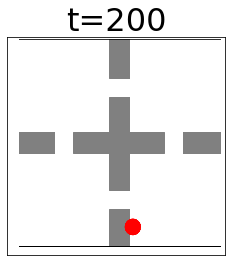}
    \\
    \includegraphics[width=0.1\textwidth]{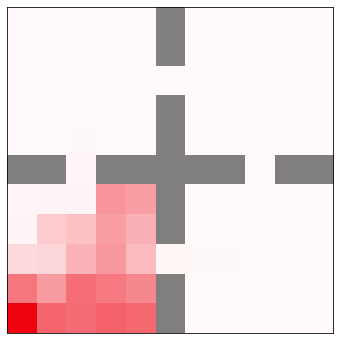}
    \includegraphics[width=0.1\textwidth]{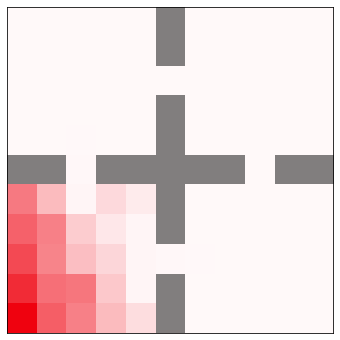}
    \includegraphics[width=0.1\textwidth]{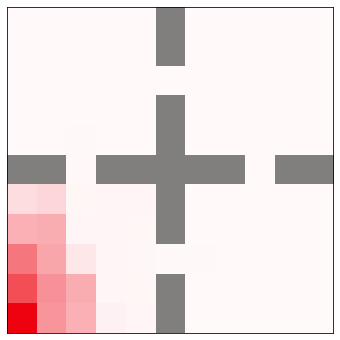}
    \includegraphics[width=0.1\textwidth]{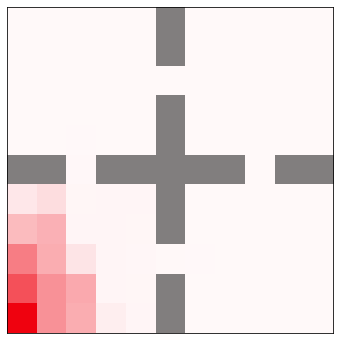}
    \includegraphics[width=0.1\textwidth]{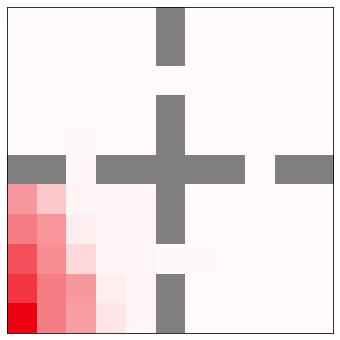}
    \includegraphics[width=0.1\textwidth]{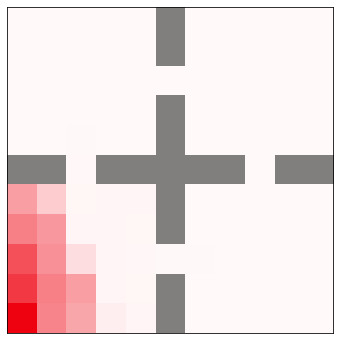}
    \includegraphics[width=0.1\textwidth]{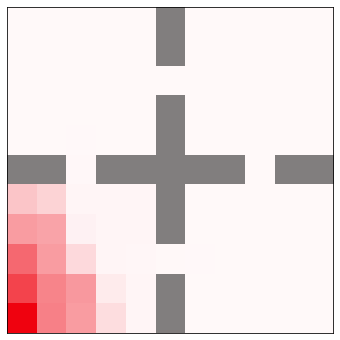}
    \includegraphics[width=0.1\textwidth]{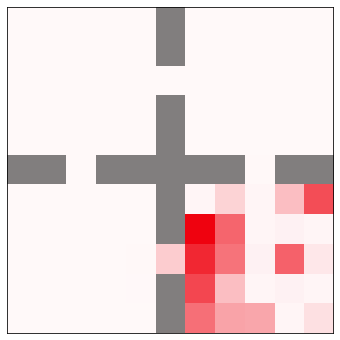}
    \includegraphics[width=0.1\textwidth]{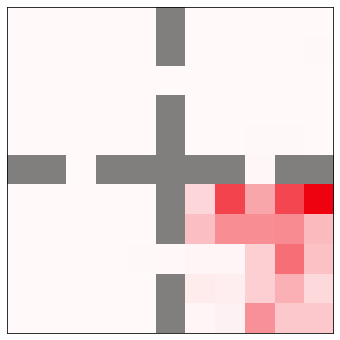}
    \caption{Beliefs for a given trial under the POMDP setting with $p=0.6$. }
    \label{fig:POMDP_belief6}
\end{figure}
\begin{figure}[H]
    \centering
    \includegraphics[width=0.1\textwidth]{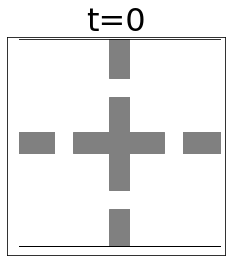}
    \includegraphics[width=0.1\textwidth]{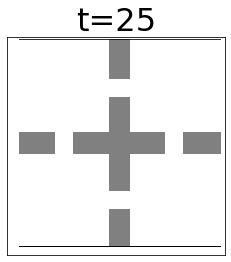}
    \includegraphics[width=0.1\textwidth]{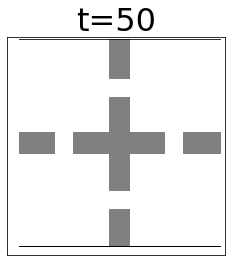}
    \includegraphics[width=0.1\textwidth]{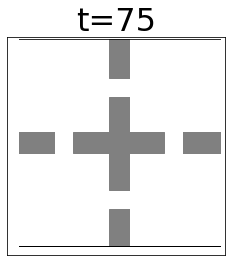}
    \includegraphics[width=0.1\textwidth]{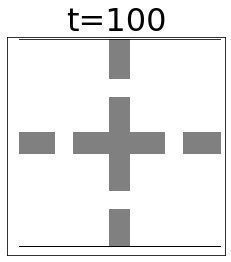}
    \includegraphics[width=0.1\textwidth]{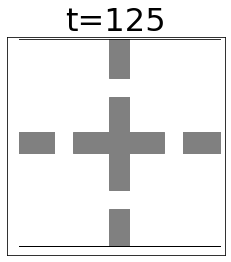}
    \includegraphics[width=0.1\textwidth]{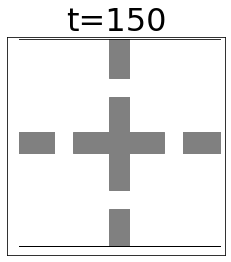}
    \includegraphics[width=0.1\textwidth]{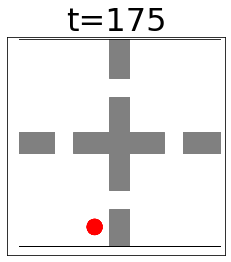}
    \includegraphics[width=0.1\textwidth]{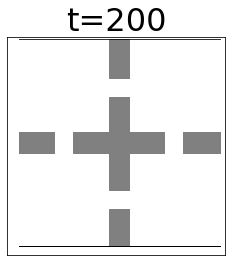}
    \\
    \includegraphics[width=0.1\textwidth]{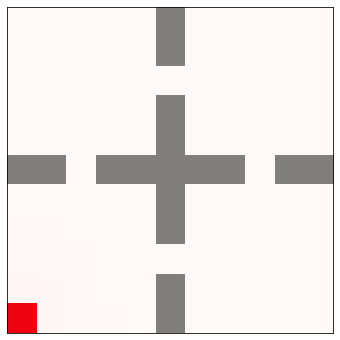}
    \includegraphics[width=0.1\textwidth]{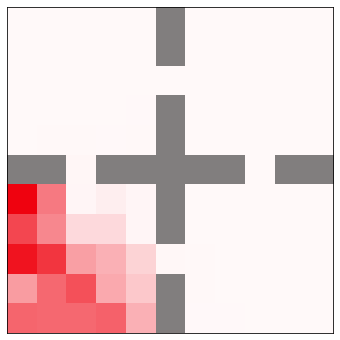}
    \includegraphics[width=0.1\textwidth]{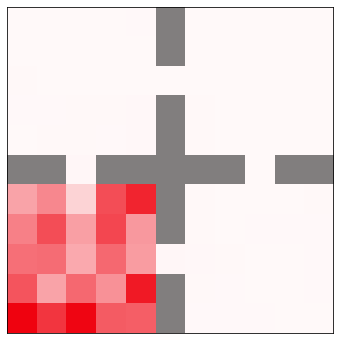}
    \includegraphics[width=0.1\textwidth]{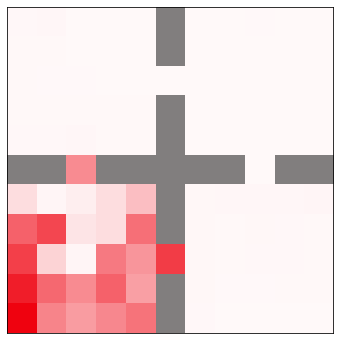}
    \includegraphics[width=0.1\textwidth]{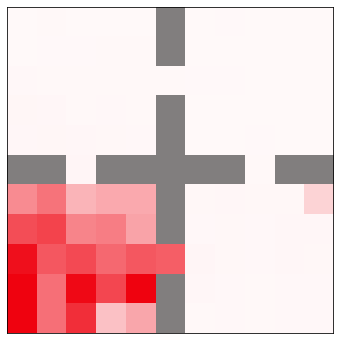}
    \includegraphics[width=0.1\textwidth]{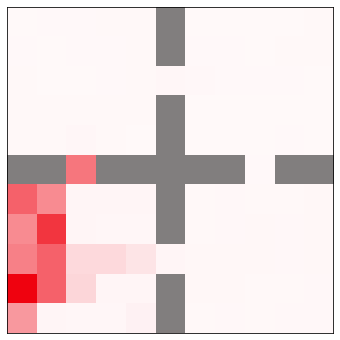}
    \includegraphics[width=0.1\textwidth]{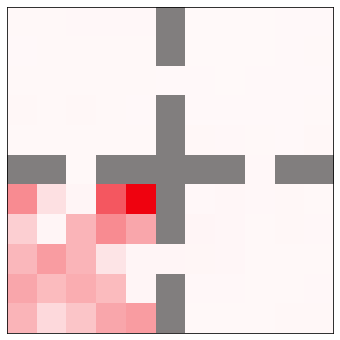}
    \includegraphics[width=0.1\textwidth]{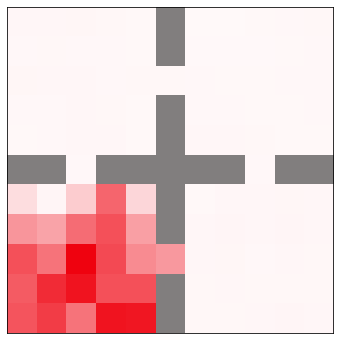}
    \includegraphics[width=0.1\textwidth]{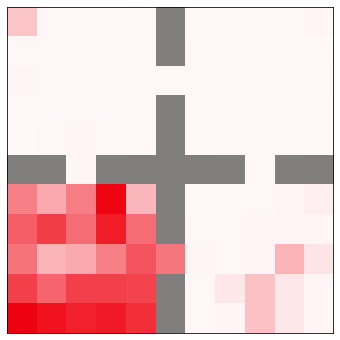}
    \\
    \includegraphics[width=0.3\textwidth]{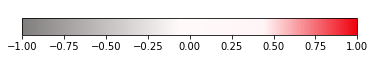}
    \caption{Beliefs for a given trial under the POMDP setting with $p=0.9$.  Top-rows illustrate the snapshots of the observations that were given to the agent at timesteps $t=\{0,25,50,75,100,125,150,175,200\}$. The bottom rows denote the belief decoded from the representation of the history till timestep $t$. See text for more details.  \textbf{Legend:} All plots on this page share the same legend. For the bottom rows, the deeper red color indicates a higher degree of belief. For the top row, the red color indicates the location of the agent. Higher $p$ indicates a higher degree of partial-observability and thus more frames have the agent missing.  }
    \label{fig:POMDP_belief9}
\end{figure}

\clearpage
\section{More Large-Scale Results}
\label{sec:moreres}

\subsection{DMLab30}

We measure the performance of each task in DMLab30 environment with the human-normalized performance
$(z_i - u_i)/(h_i - u_i)$ with $1 \leq i \leq 30$, where $u_i$ and $h_i$ are
the raw score performance of random policy and humans, and $z_i$ is the score obtained by the agent. A normalized score of 1 indicates that the agent performs as well as humans on the task.

In Figure \ref{fig:dmlab30} we plot the human normalized score aggregated across all the domains and in Figure \ref{fig:dmlab30_all} we plot the learning curves for each of the tasks individually.
It can be observed in both that when dealing with complex environments, the baseline decomposition-based method fails to provide significant improvement.
In comparison, the proposed De-Rex method has the potential to perform much better.
\begin{figure}
    \centering
    \includegraphics[width=0.4\textwidth]{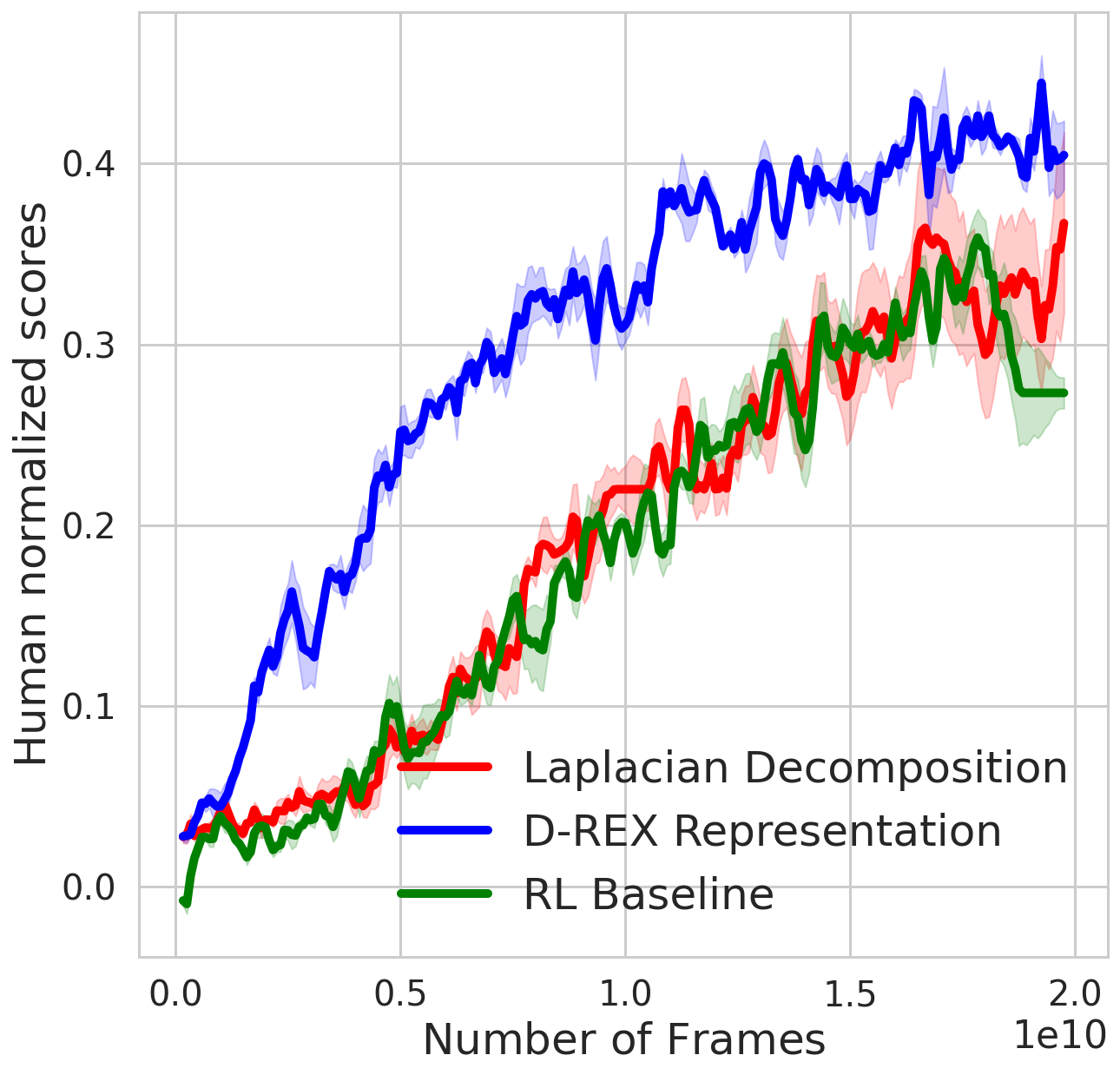}
    \caption{Human normalized scores aggregated across all the 30 tasks on DMLab30 using 3 seeds}
    \label{fig:dmlab30}
\end{figure}

\subsection{DM-Hard-8}
 
To better understand how much can just learning representations help with hard exploration tasks,
we perform ablation studies on the DM-Hard-8 tasks for the SVD representation and Laplacian decomposition method.
For both the methods we plot the results for different values of (a) $w_\text{loss}$ and (b) dimensions of the output of $f_{\theta_1}$ on which the decomposition is done.

As it can be observed from the plots in Figure \ref{fig:dmhard8_param_study} and \ref{fig:dmhard8_param_study2}, just by using representation learning alone, D-Rex is able to solve at most 2 tasks, and the Laplacian decomposition method can solve only one task.
This illustrates the importance of constructing good pseudo-counts to aid exploration, as evident in Figure \ref{fig:dmhard_base} by the De-Rex Exploration method.

\begin{figure*}
    \centering
    \includegraphics[width=1\textwidth]{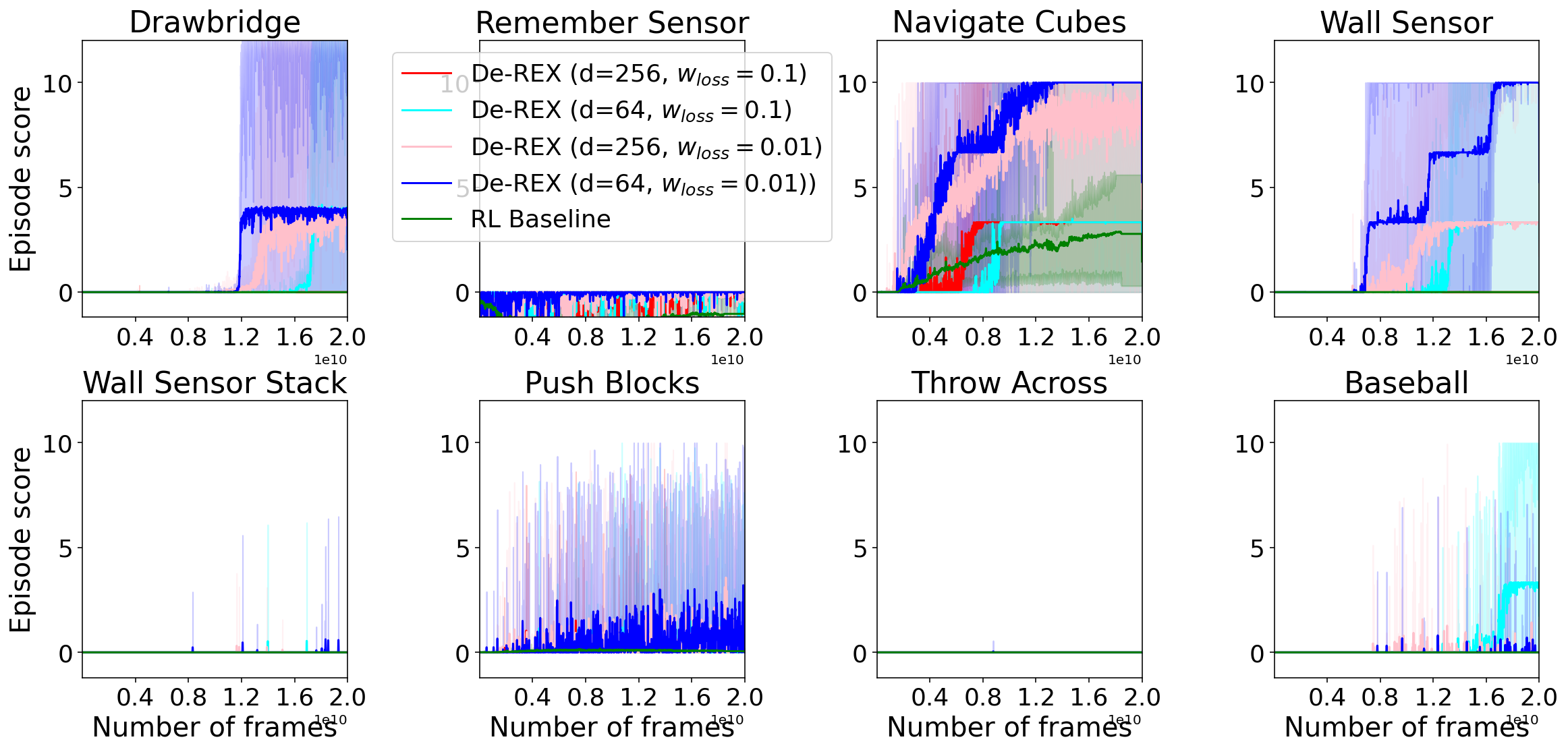}
    \caption{Parameter Study: De-Rex Representation. Sensitivity on weighting the SVD loss in the representation ($w_{loss}$) and the dimension of the SVD decomposition ($d$) }
    \label{fig:dmhard8_param_study}
\end{figure*}

\begin{figure*}
    \centering
    \includegraphics[width=1\textwidth]{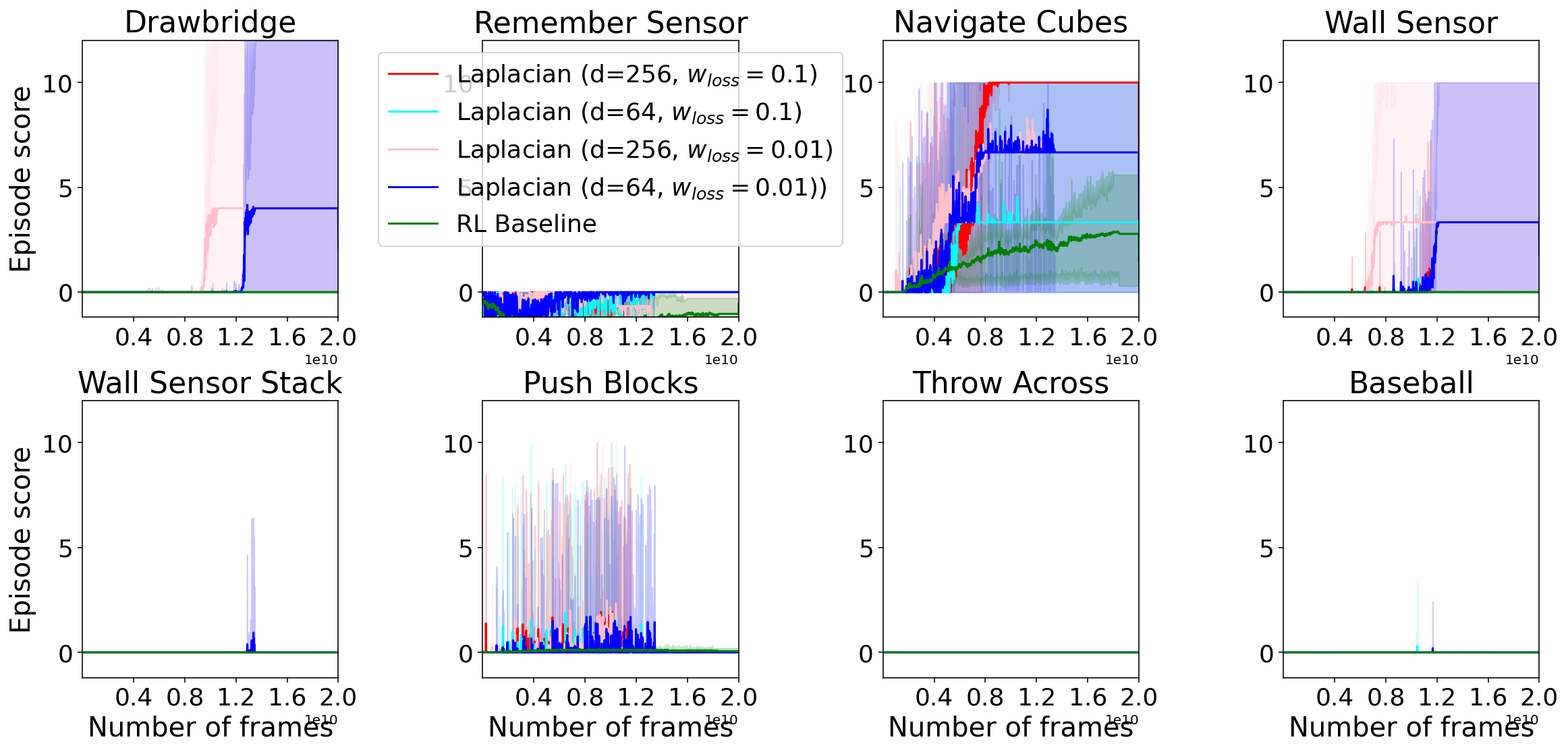}
    \caption{Parameter Study: Lapacian Decomposition Representation. Sensitivity on weighting the decomposition loss in the representation ($w_{loss}$) and the dimension of the SVD decomposition ($d$) }
    \label{fig:dmhard8_param_study2}
\end{figure*}

\begin{figure*}
    \centering
    \includegraphics[width=0.6\textwidth]{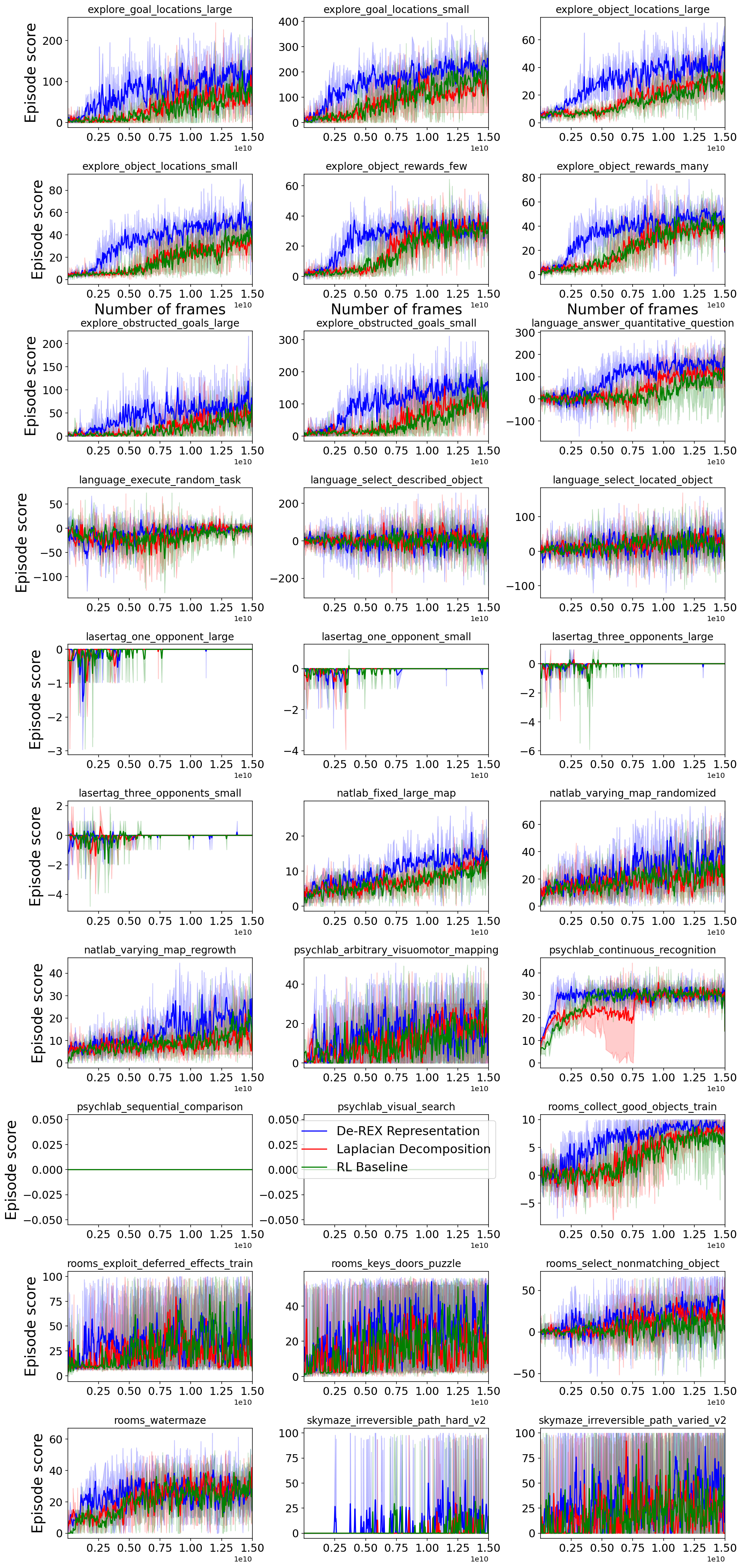}
    \caption{Individual learning curves for De-Rex Representation, Laplacian decomposition, and the VMPO baseline methods on the DMLab30 tasks. }
    \label{fig:dmlab30_all}
\end{figure*}